\documentclass[11pt]{article}
\usepackage{matus}
\def\tW{\widetilde W}
\def\RE{\textup{Re}}
\usepackage{bm}
\SetSymbolFont{largesymbols}{bold}{OMX}{txex}{b}{n}
\newcommand{\bbm}[1]{{\blue{\bm{#1}}}}

\title{Neural tangent kernels, transportation mappings, and\\{}universal approximation}

\author{Ziwei Ji\qquad{}Matus Telgarsky\qquad{}Ruicheng Xian\\
\texttt{\{\href{mailto:ziweiji2@illinois.edu}{ziweiji2},\href{mailto:mjt@illinois.edu}{mjt},\href{mailto:rxian2@illinois.edu}{rxian2}\}@illinois.edu}\\
University of Illinois, Urbana-Champaign}
\date{}

\begin{document}

\maketitle

\begin{abstract}
  This paper establishes rates of universal approximation for the shallow
  \emph{neural tangent kernel (NTK)}:
  network weights are only allowed microscopic changes from random initialization,
  which entails that activations are mostly unchanged, and the network
  is nearly equivalent to its linearization.
  Concretely, the paper has two main contributions:
  a generic scheme to approximate functions with the NTK by sampling
  from \emph{transport mappings}
  between the initial weights and their desired values,
  and the construction of transport mappings via Fourier transforms.
  Regarding the first contribution, the proof scheme provides another perspective
  on how the NTK regime arises from rescaling: redundancy in the weights due to resampling
  allows individual weights to be scaled down.
  Regarding the second contribution, the most notable transport mapping asserts that
  roughly $\sfrac 1 {\delta^{10d}}$ nodes are sufficient to approximate continuous functions, where
  $\delta$ depends on the continuity properties of the target function.
  By contrast, nearly the same proof yields a bound of
  $\sfrac 1 {\delta^{2d}}$ for shallow ReLU networks;
  this gap suggests a tantalizing direction for future work,
  separating shallow ReLU networks and their linearization.
\end{abstract}

\section{Main result and overview}

Consider functions computed by a single ReLU layer, meaning
\begin{equation}
  x\mapsto \sum_{j=1}^m s_j \sigma\del{ \ip{w_j}{ x }+ b_j },
  \label{eq:sn:1}
\end{equation}
where $\sigma(z) := \max\{0,z\}$.
While shallow networks are celebrated
as being \emph{universal approximators}
\citep{cybenko,funahashi,nn_stone_weierstrass} --- they approximate continuous
functions arbitrarily well over compact sets --- what is more shocking is that
gradient descent can learn the parameters to these networks, and they generalize \citep{zhang_gen}.

Working towards an understanding of gradient descent on shallow (and deep!) networks,
researchers began investigating the \emph{neural tangent kernel (NTK)}
\citep{jacot_ntk,du_iclr,allen_deep_opt},
which replaces a network with its linearization at initialization, meaning
\begin{equation}
  x\mapsto \frac {\eps}{\sqrt{m}} \sum_{j=1}^m s_j \ip{\tau_j}{\tx}\sigma'\del{\ip{\tw_j}{\tx }},
  \quad
  \textup{where }
\tilde x = (x,1)\in\R^{d+1},
  \ {}
  \tilde w = (w,b)\in\R^{d+1};
  \label{eq:ntk:1}
\end{equation}
here each $\tw_j = (w_j,b_j)$ is frozen at Gaussian initialization (henceforth the bias
is collapsed in for convenience),
and each \emph{transported weight} $\tau_j$ is microscopically close to the corresponding initial weight $\tw_j$,
concretely $\|\tau_j - \tw_j \| = \cO(\sfrac {1}{\eps\sqrt m})$,
where $\epsilon>0$ is a parameter
and the scaling $\sfrac {\eps}{\sqrt{m}}$ is conventional in this literature
\citep{2018arXiv181104918A}.

As \cref{eq:ntk:1} is merely \emph{affine} in the parameters, it is not outlandish that gradient
descent can be analyzed.  What \emph{is} outlandish is firstly
that
gradient descent on
\cref{eq:sn:1} with small learning rate will track the behavior of \cref{eq:ntk:1},
and
secondly
the weights hardly change \emph{as a function of $m$}, specifically $\|\tau_j - \tw_j \|_2 = \cO(\sfrac 1 {\eps\sqrt{m}})$.

\paragraph{Contributions.}  This work provides rates of function approximation
for the NTK as defined in \cref{eq:ntk:1}, moreover in the ``NTK setting'':
the transported weights must be near initialization,
meaning $\|\tau_j- \tw_j\| =  \tcO(\sfrac 1 {\eps\sqrt{m}})$.
In more detail:
\begin{description}
  \item[Continuous functions (cf. \Cref{fact:main}).]
    The main theorem packages the primary tools here to say: the NTK can approximate continuous
    functions arbitrarily well, so long as the width is at least $\sfrac {1}{\delta^{10d}}$,
    where $\delta$ depends on continuity properties of the target function;
    moreover, the transports satisfy $\|\tau_j-\tw_j\| = \cO(\sfrac {1}{\eps\sqrt m})$,
    and the ReLU network (\cref{eq:sn:1}) and its linearization (\cref{eq:ntk:1}) stay close.
    Re-using many parts of the proof, a nearly-optimal rate $\sfrac {1}{\delta^{2d}}$
    is exhibited for ReLU networks in \Cref{fact:cont:direct}; this gap between ReLU networks
    and their NTK poses a tantalizing gap for future work.

  \item[Approximation via sampling of transport mappings.]
    The first component of the proof of \Cref{fact:main}, detailed in \Cref{sec:samp},
    is a procedure which starts with an infinite width network, and describes how
    sampling introduces redundancy in the weights, and automatically
    leads to the desired microscopic transports $\|\tau_j-\tw_j\| = \cO(\sfrac{1}{\eps\sqrt m})$.
    As detailed in \Cref{fact:samp}, the error between the infinite width and sampled networks is
    $\tcO(\eps+\sfrac {1}{\sqrt m})$.  In this way, the analysis provides another perspective
    on the scaling behavior and small weight changes of the NTK.

  \item[Construction of transport mappings.]
    The second component of the proof of \Cref{fact:main}, detailed in \Cref{sec:transport},
    is to construct explicit transport mappings for various types of functions.
    As detailed in \Cref{fact:transport:fourier:gaussian},
    approximating continuous functions proceeds by constructing an infinite width network
    not directly for the target function $f$, but instead its convolution $f * G_\alpha$ with
    a Gaussian $G_\alpha$ with tiny variance $\alpha^2$.
    Care is needed in order to obtain a rate of the form $\sfrac {1}{\delta^{\cO(d)}}$,
    rather than, say, $\sfrac {1}{\delta^{\cO(d/\delta)}}$.
    The main constructions are based on Fourier transforms.

\end{description}

Rounding out the organization of this paper: this introduction will state the main summarizing
result and its intuition, and then close with related work;
\Cref{sec:cont} will describe certain odds and ends for approximating continuous functions
which were left out from the main tools in \Cref{sec:samp} and \Cref{sec:transport};
\Cref{sec:rkhs} sketches abstract approaches to constructing transport mappings, including
ones based on a corresponding Hilbert space;
\Cref{sec:open} will conclude with open problems and related discussion.  Proofs are sketched in
the paper body, but details are deferred to the appendices.

\subsection{Basic notation, intuition, and main result}

The NTK views networks as finite width realizations of intrinsically \emph{infinite width}
objects.  In order to transport an infinite number of parameters away from their initialization,
one option is to use a \emph{transport mapping} $\cT:\R^{d+1} \to \R^{d+1}$ to show
where weights should go:
\[
  x
  \ \mapsto\ {}
  \bbE_{\tw} \ip{ \cT(\tw) }{ \Phi(x;\tw) }
  \ =\ {}
  \bbE_{\tw} \ip{\cT(\tw)}{\tx} \sigma'(\ip{\tw}{\tx})
  \ =\ {}
  \bbE_{\tw} \ip{\cT(\tw)}{\tx} \1\sbr{ \ip{\tw}{\tx} \geq 0 },
\]
where $\Phi(x;\tw) = \tx \sigma'(\ip{\tw}{\tx})$ is a \emph{random feature} representation
of $x$ \citep{rahimi_recht}.  This abstracts the individual transported weights
$(\tau_j)_{j=1}^m$ from before into transported weights defined over arbitrary weights $\tw\in\R^{d+1}$.
These (augmented) weights $\tw = (w,b)$ (with weight $w\in\R^d$ and bias $b\in \R$) will always
be distributed according to a standard Gaussian with identity covariance, with $G$ denoting
the density and probability law simultaneously.

A key message of this work, developed in \Cref{sec:samp}, is (a) the infinite width network
can be sampled to give rates of approximation by finite width networks, (b) the microscopic
adjustments of the NTK setting arise naturally from the sampling process!  Indeed,
letting $s\in \{-1,+1\}$ denote a uniformly distributed random sign,
\begin{equation}
  \begin{aligned}
    \E_{\tw} \ip{ \cT(\tw) }{ \Phi(x;\tw) }
    &=
    \E_{\tw,s} \ip{ s^2  \cT(\tw) + s \tw \bbm{\eps \sqrt{m}} }{ \Phi(x;\tw) }
    &\because \E s^2 = 1, \E s = 0
    \\
    &\approx
    \bbm{ \frac 1 m } \sum_{j=1}^m \ip{ s_j^2 \cT(\tw_j) + s_j\tw_j \bbm{ \eps \sqrt{m} }}{ \Phi(x;\tw_j) }
    &\textup{sampling } (w_j, s_j)
    \\
    &=
    \bbm{ \frac \eps {\sqrt m} }
    \sum_{j=1}^m \ip{ \frac{s_j \cT(\tw_j)}{\bbm{ \eps \sqrt m}} + \tw_j }{ s_j \Phi(x;\tw_j) }.
  \end{aligned}
  \label{eq:samp}
\end{equation}
As highlighted by the bolded terms:
increasing the width $m$ corresponds to resampling,
and allows the transported weights to be scaled down!  Indeed, the distance moved is 
$\cO(\sfrac 1 {\eps\sqrt m})$ \emph{by construction}.  To this end, for convenience define
\begin{equation}
  \begin{aligned}
    \tau_j &:= \cT_{\bbm{\eps}}(\tw_j,s_j) := \frac {s_j \cT(\tw_j)}{\bbm{\eps \sqrt m}} + \tw_j\1[\|\tw_j\|\leq R],
    \\
    \phi_j(x) &:= \Phi_{\bbm{\eps}}(x;\tw_j,s_j) := \frac {\bbm{\eps} s_j}{\bbm{\sqrt{m}}} \Phi(x;\tw_j)
    = \frac {\bbm{\eps} s_j}{\bbm{\sqrt{m}}} \tx \sigma'(\ip{\tw_j}{\tx}),
  \end{aligned}
  \label{eq:defn:tau}
\end{equation}
where $R$ is a truncation radius included for purely technical reasons.
The transport mappings constructed in \Cref{sec:transport} satisfy
$B := \sup_{\tw} \|\cT(\tw)\| <\infty$, and thus $\max_j \|\tau_j - \tw_j\| \leq \sfrac {B}{\eps\sqrt{m}}$ by construction as promised
(with high probability).

The key message of \Cref{sec:samp} is to control
the deviations of this process, culminating in \Cref{fact:samp} and also \Cref{fact:main} below,
which yields upper bounds on the width necessary to approximate infinite width networks.
The notion of approximation here will follow \citep{barron_nn} and use the $L_2(P)$ metric,
where $P$ is a probability measure on the ball $\{ x\in \R^d : \|x\|\leq 1\}$:
\[
  \enVert{ h }_{L_2(P)} = \sqrt{ \int h(x)^2 \dif P(x) }.
\]
Additionally $\|h\|_{L_2} = \sqrt{\int h(x)^2\dif x}$ and $\|h\|_{L_1} = \int |h(x)|\dif x$ will respectively denote
the usual $L_2$ and $L_1$ metrics over functions on $\R^d$.

\begin{theorem}[name={Simplification of \Cref{fact:samp,fact:cont:transport}}]
  \label{fact:main}
  Let continuous function $f:\R^d\to\R$ be given, along with $\delta \in(0,1]$
  so that $|f(x) - f(x')|\leq \eps$ whenever $\|x-x'\|_2 \leq \delta$
  and $\max\{\|x\|,\|x'\|\}\leq 1+\delta$.  Let $P$ be any probability distribution
  over $\|x\|\leq1$.
  Then there exists a transport mapping $\cT$
  (defining $\cT_\eps$ and $\tau_j$ as in \cref{eq:defn:tau})
  and associated scalars
  \[
B := \sup_{\tw} \|\cT(\tw)\|_2 = 
    \tcO\del{
      \frac{M^5 d^{(5d+9)/2}}{\eps^4 \delta^{5(d+1)}}
    },
    \hspace{3em}
    \textup{where}\ {}
    M := \sup_{\|x\|\leq 1+\delta} |f(x)|,
\]
  so that with probability at least $1-3\eta$ over
  Gaussian weights $(\tw_j)_{j=1}^m$
  and uniform signs $(s_j)_{j=1}^m$,
  then
  $\max_j \|\tau_j - \tw_j\| \leq \sfrac {B}{\eps\sqrt m}$, and
  \begin{align*}
    \enVert{
f
      -
      \sum_{j=1}^m \ip{\tau_j}{\phi_j(\cdot)}
    }_{L^2(P)}
    &\leq
    \tcO\del{
    \sbr{ \frac B {\sqrt m} + \eps \sqrt{d} }
    \sqrt{\ln(1/\eta)} },
    \\
    \enVert{
      \frac{\eps}{\sqrt m} 
      \sum_{j=1}^m s_j \sigma(\ip{\tau_j}{\tx})
      -
      \sum_{j=1}^m \ip{\tau_j}{\phi_j(\cdot)}
    }_{L_2(P)}
    &\leq
    \tcO\del{
      \sbr{\frac {B^2}{\eps m^{3/2}} + \frac {B\sqrt{d}}{m} +  \frac B {\sqrt m} + \eps \sqrt{d} }
      \sqrt{\ln(1/\eta)}
    }.
  \end{align*}
\end{theorem}

In words: given an arbitrary target function $f$ and associated continuity
parameter $\delta$, width $(B/\eps)^2 = \tcO(\sfrac{d^{5d+9}}{\eps^{10} \delta^{10(d+1)}})$ suffices for error
$\tcO(\eps)$, parameters are close to initialization, and the NTK and the
original network behave similarly.  The randomized construction does not merely give
existence, but holds with high probability:
the sampling process is thus in a sense robust, and may be used algorithmically!

As provided in \Cref{fact:cont:direct:simple}, elements of the proof of \Cref{fact:main}
can be extracted and converted into a direct approximation rate of continuous functions by ReLU
networks, and the rate becomes $\tcO(\sfrac {d^{d+2}}{\eps^2 \delta^{2d+2}})$.  
Since this rate is nearly tight, together these rates pose an interesting question:
is there a purely approximation-theoretic gap between shallow ReLU networks
and their NTK?

\subsection{Related work}

\paragraph{Optimization literature; the NTK.}
This work is motivated and inspired by the optimization literature, which introduced the NTK
to study gradient descent in a variety of nearly-parallel works
\citep{jacot_ntk,du_iclr,du_deep,allen_deep_opt,arora_2_gen,
    oymak_moderate,li_liang_nips,2019arXiv190201384C}.
These works require the network width to be polynomial in $n$, the size of the training set;
by contrast, the analysis here studies closeness in function space,
and the width instead scales with properties of the target function.

One close relative to the present work is that of \citep{chizat_bach_note},
which provides an abstract proof scheme following the preceding works, and explains the
microscopic change of the weights as a consequence of the scaling $\sfrac {\eps}{\sqrt m}$.
This is consistent with the resampling perspective here, as summarized in \cref{eq:samp}.

\paragraph{Random features and the mean-field perspective.}
The \emph{random features} perspective \citep{rahimi_recht} studies a related \emph{convex}
problem: similarly to the NTK, the activations $\sigma'(\ip{\tw_j}{\tx})$ are held fixed,
and what are trained are scalar weights $a_j\in\R$ on each feature.
The Fourier transport map construction used both for the NTK here in \Cref{fact:main}
and for shallow networks in \Cref{fact:cont:direct} proceeds by constructing
exactly such a reweighting, and thus the present work also establishes universal approximation
properties of random features.  A related perspective is presented in the mean-field
literature,
which relate gradient descent on $(\tw_j)_{j=1}^m$ 
to a Wasserstein flow in the space of distributions on these features
\citep{chizat_bach_meanfield,montanari_meanfield}.  The analysis here does not have any
explicit ties to the mean-field literature,
however it is interesting and suggestive that transport mappings appear in both.

\paragraph{Approximation literature.}
The closest prior work is due to \citet{barron_nn}, who gave good rates of approximation
for functions $f:\R^d\to\R$ when the associated quantity
$\int \|w\|\cdot |\hat f(w)| \dif w$
is small, where $\hat f$ denotes the Fourier transform of $f$.
The proofs in \Cref{sec:transport} will use elements from the proofs in \citep{barron_nn},
but with many distinct components, and thus it is interesting that the same quantity
$\int \|w\|\cdot |\hat f(w)| \dif w$ arises once again.
Like the work of \citep{barron_nn}, the present work also chooses to approximate in the
$L_2(P)$ metric.
Standard classical works in this literature are general universal approximation guarantees
without rates or attention to the weight magnitudes
\citep{cybenko,nn_stone_weierstrass,funahashi,leshno_apx}.
The rate given here of roughly $\sfrac{1}{\delta^{2d}}$ in \Cref{fact:cont:direct:simple}
does not seem to appear rigorously in prior
work, though it is mentioned as a consequence of a proof in 
\citep{mhaskar1992approximation}, who also take the approach of approximation via Gaussian
convolutions;
the use of convolutions is not only standard \citep{wendland_2004}, it is moreover classical,
having been used in Weierstrass's original proof \citep{weierstrass_apx}.

Many related works use a RKHSes directly.
\citet{tewari_sun_gilbert_rkhs} prove universal approximation (with rates)
via an RKHS, however they do not consider the NTK (or the NTK setting of small weight changes).
\citet{bach_convexified} (see also \citep{bach_quadrature,basri,bietti})
studies a variety of questions related to function fitting with the
random features model, including establishing rates of approximation for Lipschitz functions
on the surface of the sphere (with a few further conditions);
the rates are better than those here (roughly $\Theta(\sfrac 1 {\delta^{d/2}})$),
however they do not consider the NTK setting,
meaning either the setting of small changes from initialization
nor the linearization.

Another close parallel work studies exact representation power of infinite width networks,
developing representations for functions with $\Omega(d)$ derivatives \citep{nati_poly_bumps};
similarly, the constructions here use an exact representation result for Gaussian convolutions,
as developed in \Cref{sec:transport}.

Regarding lower bounds from the literature,
there are two lower bounds of the form
$\sfrac{1}{\delta^{d/2}}$ for general shallow networks, not necessarily
in the NTK setting \citep{yarotsky,bach_convexified}.
Interestingly, \citet{yarotsky} also presents a lower bound of $\sfrac 1 {\delta^d}$
for approximations whose
parameters vary \emph{continuously} with the target function; this seems to hold for the Fourier
constructions here in \Cref{sec:transport}, though an argument needs to be made for the sampling
step.

\section{Sampling from a transport}
\label{sec:samp}

This section establishes that by sampling from an infinite width NTK, the resulting finite width
NTK is close in $L_2(P)$ both to the infinite width idealization, and also to the finite width
non-linearized ReLU network; moreover, the sampling process introduces redundancy in the weights, allowing them to be
scaled down and lie close to initialization.

\begin{theorem}
  \label{fact:samp}
  Suppose $B\geq \max\cbr{ 2, \sup_{\tw}\|\cT(\tw)\|_2}$,
  and 
  set $R := \sqrt{d+1} + 2\sqrt{\ln(m/\eta)}$.
  With probability at least $1-3\eta$,
  then
  $\max_j \|\tau_j - \tw_j\| \leq \sfrac {B}{\eps\sqrt m}$, and
  \begin{align}
    \enVert{\sum_{j=1}^m \ip{\tau_j}{\phi_j(\cdot)}
      - \bbE_{\tw} \ip{\cT(\tw)}{\Phi(\cdot;\tw)}
    }_{L^2(P)}
    &\hspace{-1em}\leq
    2\del{ \frac B {\sqrt m} + \eps R}
    \sbr{ 1 + \sqrt{\ln(1/\eta)} },
    \label{eq:samp:1}
    \\
    \enVert{
      \sum_j \ip{\tau_j}{\phi_j(\cdot)}
      -\sum_j
 \frac {s_j \eps}{\sqrt m}
      \sigma(\ip{\tau_j}{\cdot})
    }_{L_2(P)}
    &\hspace{-1em}\leq
    2\del{\frac {B^2}{\eps m^{3/2}} + \frac {BR}{m} +  \frac B {\sqrt m} + \eps R}
    \sbr{
      1 + \sqrt{\ln(1/\eta)}
    }.
    \label{eq:samp:2}
  \end{align}
\end{theorem}

As discussed in the introduction, $\max_j \|\tau_j - \tw_j\| \leq \sfrac {B}{\eps\sqrt m}$
is essentially by construction.  Next, recall the sampling derivation in \cref{eq:samp},
restated here as a lemma for convenience, the notation $(\cW,S)$ collecting all random variables together, meaning
$\cW = (\tw_1,\ldots,\tw_m)$ and $S = (s_1,\ldots, s_m)$.

\begin{lemma}
  \label{fact:sampling}
  $\displaystyle
  \E_{\tw} \ip{ \cT(\tw) }{ \Phi(x;\tw) }
= \E_{\tW,S} \sum_j \ip{\cT_\eps(\tw_j,s_j)}{ \Phi_\eps(x; \tw_j, s_j) }
= \E_{\tW, S} \sum_j \ip{\tau_j}{\phi_j(x)}
  $.
\end{lemma}

The proof of \cref{eq:samp:1} now follows from the classical Maurey sampling lemma
\citep{pisier1980remarques}, which was also used in the related work by \citet{barron_nn}.
The following version additionally includes a high probability
control, which results from an application of McDiarmid's inequality.
Applying the following sampling lemma to the present setting, the deviations will scale
with $B := \sup_{\tw} \|\cT(\tw)\|_2$.

\begin{lemma}[Maurey]
  \label{fact:maurey}
  Let functions $\{ g(\cdot;v) : v \in \cV \}$ be given,
  where $\cV\subseteq \R^p$ is a set of possible parameters.
  Let $\nu$ be a probability measure over $\cV$,
  let $(v_1,\ldots,v_m)$ be an iid random draw from $\nu$,
  and define
  \[
    f(x) := \E_{v\sim\nu} g(x;v)
    \qquad\textup{and}\qquad
    g_j(x) := g(x;v_j).
  \]
  Then
  \[
    \E_{((s_j,v_j))_{j=1}^m}
    \enVert{ f
      - \frac 1 m \sum_{j=1}^m g_j
    }_{L^2(P)}^2
    \leq
    \frac 1 m 
    \E_v \enVert{ g(\cdot;v) }_{L_2(P)}^2
    \leq
    \frac 1 m 
    \sup_{v\in\cV} \enVert{ g(\cdot;v) }_{L_2(P)}^2,
  \]
  and with probability at least $1-\eta$,
  \begin{align*}
    \enVert{ f
      - \frac 1 m \sum_{j=1}^m g_j }_{L^2(P)}
&\leq
    \sup_{v\in\cV} \|g(\cdot;v)\|_{L_2(P)}
    \sbr{ \frac{ 1 + \sqrt{2\ln(1/\eta)}}{\sqrt {m}} }.
  \end{align*}
\end{lemma}

Concretely, here $g_j(x) = m \ip{\tau_j}{\phi_j(x)}$,
and
$\sup_{v\in\cV} \|g(\cdot; \tw)\|_{L_2(P)} \leq \sqrt{2} \sup_{\tw} \|\cT(\tw)\|_2
= \cO( B + R\eps\sqrt{m} )$ by Cauchy-Schwarz.
Before continuing, note also that there are other proof schemes attaining similar
bounds \citep[Proposition 1]{bach_convexified,bach_quadrature}, and that similar bounds
are possible for the uniform norm, albeit with more sensitivity to the basis functions $g$
(cf. \Cref{fact:sample:uniform}).

Turning now to the final bound in \cref{eq:samp:2},  the first step is to note by
positive homogeneity of the ReLU that
$\sigma(\ip{\tau_j}{\tx}) = \ip{\tau_j}{\tx} \sigma'(\ip{\tau_j}{\tx})$, thus
\[
  \sum_{j=1}^m \ip{\tau_j}{\phi_j} - \frac {\eps}{\sqrt m} \sum_{j=1}^m s_j \sigma(\ip{\tau_j}{\tx})
= \sum_{j=1}^m \ip{\tau_j}{\phi_j - \frac{s_j \eps}{\sqrt m} \tx \sigma'(\ip{\tau_j}{\tx})},
\]
which boils down to checking the difference in activations, namely
$\sigma'(\ip{\tw_j}{\tx}) - \sigma'(\ip{\tau_j}{\tx})$.  As is standard in the NTK literature,
since $\tau_j-\tw_j$ is (with high probability) microscopic compared to $\ip{\tw_j}{\tx}$,
the activations should also be close.  The following lemma makes this precise.

\begin{lemma}
  \label{fact:activations}
  For any $x\in\R^p$,
  if $R \geq \sqrt{d} + 2 \sqrt{\ln\del{ \frac {\eps\sqrt{m\pi}}{B \sqrt{2}}}}$
  (as used in \cref{eq:defn:tau}),
  then 
\[
    \E_{\tw}
    \envert{ \sigma'(\ip{\tw}{\tx}) - \sigma'(\ip{\cT_\eps(\tw)}{\tx}) } \leq
    \frac{2B\sqrt{2}}{\eps\sqrt{m\pi}}.
  \]
\end{lemma}

From here, the \cref{eq:samp:2} can be established with another application of
\Cref{fact:maurey}.  This completes the proof of \Cref{fact:samp} after an application
of Gaussian concentration to ensure $\max_j \|\tw_j\| \leq R$.
This also establishes the first half of \Cref{fact:main}.

\section{Concrete transport mappings via Fourier transforms}
\label{sec:transport}

The previous section showed function approximation in the NTK setting \emph{assuming} the existence
of an infinite width NTK defined by a transport mapping $\cT$;
this section will close the gap by providing a variety of transport maps.

The transport mappings here will be constructed via Fourier transforms, with convention
\[
  \hat f(x) = \int \exp\del{-2\pi i x^\T w} f(x) \dif x;
\]
a few general properties are summarized in \Cref{app:technical}.
Interestingly, these transports are all \emph{random feature transports}:
they have the form $\cT(\tw) = (0,\cdots,0,p(\tw))$ where $p$ is a \emph{signed density}
over random features, and
$\bbE_{\tw}\ip{\cT(\tw)}{\Phi(x;\tw)} = \bbE_{\tw} p(\tw) \sigma'(\ip{\tw}{\tx})$.
This perspective of a signed density will be used to prove universal approximation
--- again via sampling! --- 
of shallow ReLU networks (and random features) later in
\Cref{fact:cont:direct:simple,fact:cont:direct}.
(For constructions which are not based on random features, see \Cref{sec:rkhs}.)

The first steps of the approach here follow a derivation due to \citet{barron_nn}.
Specifically, the \emph{inverse} Fourier transform gives a way to rewrite a function
as an infinite with network with complex-valued activations $x\mapsto \exp(2\pi i x^\T w)$:
\[
  f(x) = \int \exp(2\pi i x^\T w) \hat f(w) \dif w.
\]
A key trick due to \citet{barron_nn} is to force the right hand side to be real (since the left
hand side is real): specifically, letting $\hat f(w) = |\hat f(w)| \exp(2\pi i \theta_f(w))$
with $|\theta_f(w)|\leq 1$ denote the radial decomposition of $\hat f$,
\begin{align*}
  \RE f(x)
  &= \RE \int \exp(2\pi i x^\T w) \hat f(w) \dif w
  \\
  &= \RE \int \exp(2\pi i x^\T w + 2\pi i \theta_f(w) ) | \hat f(w) | \dif w
  \\
  &= \int \cos\del{ 2\pi(x^\T w + \theta_f(w) ) } | \hat f(w) | \dif w.
\end{align*}
After this step, the proofs diverge: the approach here is to use the fundamental 
theorem of calculus to rewrite $\cos$ in terms of $\sigma'$:
\[
  \cos(z) - \cos(0) = - \int_0^z \sin(b) \dif b = - \int_0^\infty \sin(b) \1[z - b \geq 0]\dif b,
  = - \int_0^\infty \sin(b) \sigma'(z - b)\dif b;
\]
plugging this back in gives an explicit representation of $f$ in terms of an infinite width
threshold network!  A similar approach can be used to obtain an infinite width ReLU network.

This is summarized in the following lemma, which includes a calculation of the error incurred by
truncating the weights; this truncation is necessary when applying the sampling of \Cref{sec:samp}.
Interestingly, this truncation procedure leads to the quantity $\int \|w\|\cdot |\hat f(w)|\dif w$,
which was \emph{explicitly} introduced as a key quantity by \citet{barron_nn}
via a different route, namely of introducing a factor $\|w\|$ to enforce decay on $\cos$.

\begin{lemma}
  \label{fact:fourier}
  Let $f:\R^d\to\R$ be given with Fourier transform $\hat f$
  and truncation radius $r \in [0,\infty]$.
  \begin{enumerate}
    \item
      Define infinite width threshold network
      \begin{align*}
        F_r(x)
        &:= f(0) + \int |\hat f(w) | \cos\del{2\pi (\theta_f(w) - \|w\|) } \dif w
        \\
        &\quad
        +
        2\pi \int \sigma'(\ip{\tw}{\tx})|\hat f(w)|
        \sin(2\pi(\theta_f(w) - b))  \1[ |b| \leq \|w\|\leq r]
        \dif \tw.
      \end{align*}
      For any $\|x\|\leq 1$,
      $F_\infty = f$ and $\envert{ f(x) - F_r(x) } \leq
      4\pi \int_{\|w\|> r} \|w\|\cdot |\hat f(w)| \dif w$.

    \item
      Define infinite width ReLU network
      \begin{align*}
        Q_r(x)
        &:=
        f(0)
        + \int |\hat f(w) | \sbr{ \cos(2\pi(\theta_f(w)-\|w\|)) -2\pi \|w\|\sin(2\pi(\theta_f(w)-\|w\|)) } \dif w
        \\
        &\quad{} + x^\T \int w |\hat f(w)|\dif w
        \\
        &\quad{}+ 4\pi^2 \int \sigma(\tw^\T \tx) |\hat f(w)| \cos(2\pi(\|w\| - b))
        \1[|b| \leq \|w\|\leq r] \dif \tw.
      \end{align*}
      For any $\|x\|\leq 1$,
      $Q_\infty = f$ and $\envert{ f(x) - Q_r(x) } \leq
        12 \pi^2 
        \int_{\|w\|>r} \|w\|^2 \cdot |\hat f(w) | \dif w$.
  \end{enumerate}
\end{lemma}

(The second part of the shows that the same technique allows functions to be written with
\emph{equality} as ReLU networks; this is included as a curiosity and used in a few places
in the appendices, but is not part of the main NTK story.)

The preceding constructions immediately yield transport mappings from Gaussian initialization
to the function $f$ in a brute-force way: by introducing the fraction $\sfrac {G(\tw)}{G(\tw)}$,
calling the numerator part of the integration measure, and the denominator part of the integrand.
As stated before, these transport maps are random feature maps: they zero out the coordinates
corresponding to $x$!

\begin{lemma}
  \label{fact:transport:fourier}
  Let $f:\R^d\to\R$ be given with Fourier transform $\hat f$.
  For any $r \in [0,\infty]$, define transport mapping $\cT_r(w,b) := (0,\ldots,0,p_r(\tw))$
  with
  \begin{align*}
    \cT_r(w,b)_{d+1} = p_r(\tw) 
    &:=
    2 \sbr{ 
      f(0) + \int |\hat f(v) | \cos(2\pi(\theta_f(v) - \|v\|))\dif v
    }
    \\
    &\quad+ 2\pi \del{ \frac{|\hat f(w)|}{G(\tw)}}
    \cos(2\pi(\theta_f(w) - b)) \1[|b| \leq \|w\| \leq r].
  \end{align*}
  By this choice, for any $\|x\|\leq 1$,
  $f(x) = \bbE_\tw \ip{ \cT_\infty(\tw)}{ \Phi(x;\tw)}$,
  and
  \begin{align*}
    \sup_{\tw} \|\cT_r(\tw)\|_2
    &\leq
    2 \envert{ f(0) }
    + 2 \int |\hat f(v) | \dif v
    + 2 \pi \sup_{\substack{\|w\|\leq r\\|b| \leq \|w\|}} \frac{|\hat f(w)|}{G(\tw)},
    \\
    \envert{ f(x) - \E \ip{\cT(\tw)}{\Phi(x;\tw)}}
    &\leq
    4\pi \int_{\|w\|>r}
    |\hat f(w)| \cdot \|w\|
    \dif w.
  \end{align*}
\end{lemma}

The preceding construction may seem general, however it is quite loose, noting the final supremum
term within $\sup_{\tw} \|\cT_r(\tw)\|_2$; indeed, attempting to plug this construction into
\Cref{fact:samp} does not yield the $\sfrac{1}{\delta^{\cO(d)}}$ rate in \Cref{fact:main}, but instead
a rate $\sfrac{1}{\delta^{\cO(d/\delta)}}$, which is disastrously larger!

Interestingly, a fix is possible for special functions of the form $f*G_\alpha$, namely convolutions
with Gaussians of coordinate-wise variance $\alpha^2$.
These are exactly the types of functions used in \Cref{sec:cont} to approximate
continuous functions.  The fix is simply to apply a change of variable so that, in a sense,
the target function and the initialization distribution have similar units.

\begin{lemma}
  \label{fact:transport:fourier:gaussian}
  Let function $f$, variance $\alpha^2 > 0$,
  and $r \in [0,\infty]$ be given,
  and define $f_\alpha := f * G_\alpha$
  and $\phi := (2\pi\alpha)^{-1}$, and transport mapping
  $\cT_r(w,b) := (0,\ldots,0,p_r(\tw))$
  with
  \begin{align*}
    \cT_r(w,b)_{d+1} = p_r(\tw)
    &:=
    2 \sbr{ 
      f_\alpha(0) + \int |\hat f_\alpha(v) | \cos(2\pi(\theta_{f_\alpha}(v) - \|v\|))\dif v
    }
    \\
    &\quad
    + 2\pi (2\pi \phi^2)^{(d+1)/2} |\hat f(\phi w)| e^{b^2/2}
    \sin(2\pi(\theta_{f_\alpha}(\phi w) - b)) \1[|b|\leq \|w\| \leq r].
  \end{align*}
  Then $f_\alpha(x) = \bbE_{\tw}\ip{\cT_\infty(\tw)}{\Phi(x;\tw)}$ for $\|x\|\leq 1$,
  and for $r \in [\sqrt{d}, \infty)$,
  \begin{align*}
    \sup_{\tw} \|\cT_r(\tw)\|
    &\leq
2 \sbr{
      M
      + (2\pi\phi^2)^{d/2} M_f
      \del{ 1 + \sqrt{2\pi^3 \phi^2} e^{r^2/2}
      }
    }
    ,
  \end{align*}
where
  $M := \sup_x |f(x)|$,
  and
  $M_f = 1$ when $f_\alpha = G_\alpha$
  and $M_f = \|f(\phi\cdot)\|_{L_1}$ otherwise, and
  \begin{align*}
    \sup_{\|x\|\leq 1} \envert{ f(x) - \bbE \ip{\cT_r(\tw)}{\Phi(x;\tw)} }
&
    \leq
    4\pi (2\pi\phi^2)^{(d+1)/2}
    M_f
(\sqrt{d}+3)
    \exp\del{ -(r-\sqrt{d})^2/4 }.
  \end{align*}
\end{lemma}

\section{Approximating continuous functions}
\label{sec:cont}

The final piece needed to prove \Cref{fact:main} is to show that a function $f$ is close
to its Gaussian convolution $f*G_\alpha$, at least when $\alpha>0$ is chosen appropriately.
This is a classical topic \citep{wendland_2004}, and indeed it was used in the original
proof of the Weierstrass approximation theorem \citep{weierstrass_apx}.
The treatment here will include enough detail necessary to yield explicit rates.

The following definition will be used to replace the usual $(\epsilon,\delta)$ conditions
associated with continuous functions with an exact quantity.

\begin{definition}
  Let $f:\R^d\to\R$ be given, and define \emph{modulus of continuity} $\omega_f$
  as
  \[
    \omega_f(\delta) := \sup\cbr{ f(x) - f(x') : \max\{ \|x\|, \|x'\|\} \leq 1+\delta, \|x-x'\|\leq \delta }.
    \qedhere
  \]
\end{definition}

If $f$ is continuous, then $\omega_f$ (defined here over a compact set) is not only finite for
all inputs, but moreover $\lim_{\delta\to 0} \omega_f(\delta) \to 0$.
It is also possible to use this definition with discontinuous functions; note additionally
that the convolution bounds in \Cref{sec:transport} only required an $L_1$ bound
on the pre-convolution function $f$, and therefore the tools throughout may be applied to
discontinuous functions, albeit with some care to their Fourier transforms!

\begin{lemma}
  \label{fact:cont}
  Let $f:\R^d\to\R$  and $\delta>0$ be given,
  and define
  \[
    M:=\sup_{\|x\|\leq 1+\delta} |f(x)|,
    \qquad
    f_{|\delta}(x) := f(x) \1[\|x\| \leq 1+\delta],
    \qquad
    \alpha := \frac {\delta}{\sqrt{d} + \sqrt{2 \ln(2M/\omega_f(\delta))}}.
  \]
  Let $G_\alpha$ denote a Gaussian with the preceding variance $\alpha^2$.
  Then
  \[
    \sup_{\|x\|\leq 1} \envert{ f - f_{|\delta}*G_\alpha } \leq 2 \omega_f(\delta).
  \]
\end{lemma}

The proof splits the integrand into two parts: points close to $x$, and points far from it.
Points close to $x$ must behave like $f(x)$ due to continuity, whereas points far from $x$ are
rare and do not matter due to the Gaussian convolution.  The full details are in the appendix.

\Cref{fact:cont} can be combined with the transport for $f*G_\alpha$ from \Cref{sec:transport}
to give a transport for approximating continuous functions.

\begin{theorem}
  \label{fact:cont:transport}
  As in \Cref{fact:cont},
  let $f:\R^d\to\R$ and $\delta>0$ be given,
  and define
  \begin{align*}
    M
    &:=\sup_{\|x\|\leq 1+\delta} |f(x)|,
    &
    f_{|\delta}(x) 
    &:= f(x) \1[\|x\| \leq 1+\delta],
    \\
    \alpha
    &:= \frac {\delta}{\sqrt{d} + \sqrt{2 \ln(2M/\omega_f(\delta))}}
    = \tcO(\sfrac {\delta}{\sqrt{d}}),
    &
    r 
    &:= \sqrt{d} + 2 \sqrt{ \ln \del{\frac {4\pi M_f (\sqrt d + 3)}{(2\pi\alpha^2)^{(d+1)/2}\omega_f(\delta)}}}.
  \end{align*}
  Let $G_\alpha$ denote a Gaussian with the preceding variance $\alpha^2$,
  and let $\cT_r$ denote the truncated Fourier map constructed in 
  \Cref{fact:transport:fourier:gaussian} for $f_{|\delta}*G_\alpha$,
  with preceding truncation choice $r$.
  Then
  \begin{align*}
    \sup_{\tw} \|\cT_r(\tw)\|
    &=
    \tcO\del{
      \|f_{|\delta}\|_{L_1}^5 \del{\frac{\sqrt{ d}}{\delta}}^{5(d+1)}
      \sbr{ \frac {\sqrt{d}}{\omega_f(\delta)} }^4
    },
    \\
    \sup_{\|x\|\leq 1} \envert{f - \bbE_\tw \ip{\cT_r(\tw)}{\Phi(x;\tw)} }
    &\leq
    \omega_f(\delta).
  \end{align*}
\end{theorem}

This completes all the pieces needed to prove \Cref{fact:main}.

\begin{proof}[Proof of \Cref{fact:main}]
  Let $f$ be given, and let $\cT_r$ denote the transport mapping provided by
  \Cref{fact:cont:transport} for $f_{|\delta}*G_\alpha$, whose various parameters
  match those in the statement of \Cref{fact:main}.
  The proof is completed by plugging $\cT_r$ into \Cref{fact:samp},
  and simplifying by noting that $\eps \geq \omega_f(\delta)$ by definition,
  and $\|f_{|\delta}\|_{L_1} = \cO(M)$ since $\delta \leq 1$.
\end{proof}

As mentioned earlier, the infinite width network constructed in \Cref{fact:fourier}
via inverse Fourier transforms can be used to succinctly prove
(via \Cref{fact:maurey} and \Cref{fact:cont}) that threshold and ReLU networks are universal
approximators, with a rate vastly improving upon that of \Cref{fact:main}.

Before stating the result, one more tool is needed: a sampling semantics
for signed densities; see also \citep{bach_convexified,bach_quadrature} for further
development and references.

\begin{definition}
  \label{defn:signed_sampling}
  A sample from a signed (Lebesgue) density $p:\R^{d+1}\to\R$ with $\|p\|_{L_1}<\infty$
  is a pair $(\tw,s)$ where $\tw$ is sampled from the probability density
  $\sfrac {|p|}{\|p\|_{L_1}}$, and $s := \sgn(p(\tw))$.
  Let $\E_p$ denote the corresponding expectation over $(\tw,s)\sim p$.
\end{definition}

This notion of signed sampling also has a corresponding Maurey lemma, and an analogue for
the uniform norm; both are provided in \Cref{app:sampling_tools}.  The full detailed universal
approximation theorems for threshold and ReLU networks are provided in \Cref{app:cont};
a simplified form for threshold networks alone is as follows.
In either case, the proof proceeds by applying signed density sampling bounds (e.g.,
appropriate generalizations of \Cref{fact:maurey}) to the infinite width networks
constructed in \Cref{fact:fourier}.
Curiously, the simplified bound stated here for threshold networks for
the uniform norm is only a multiplicative factor $\sqrt{d}$ larger than the $L_2(P)$ bound
in \Cref{fact:cont:direct}.

\begin{theorem}[name={Simplification of \Cref{fact:cont:direct}}]
  \label{fact:cont:direct:simple}
  Let $f:\R^d\to\R$ and $\delta>0$ be given,
  and define
  \[
    M:=\sup_{\|x\|\leq 1+\delta} |f(x)|,
    \qquad
    f_{|\delta}(x) := f(x) \1[\|x\| \leq 1+\delta],
    \qquad
    \alpha := \frac {\delta}{\sqrt{d} + \sqrt{2 \ln(2M/\omega_f(\delta))}}.
  \]
  Then there exist $c\in \R$ and $p :\R^{d+1}\to\R$ with
  \[
    |c| \leq M + \|f_{|\delta}\|_{L_1} (2\pi\alpha^2)^{d/2},
    \qquad\textup{and}\qquad
    \|p\|_{L_1} \leq 2\|f_{|\delta}\|_{L_1} \sqrt{\frac{2\pi d}{(2\pi\alpha^2)^{d+1}}},
  \]
  so that, with probability $\geq 1-3\eta$ over $((s_j,\tw_j))_{j=1}^m$
  drawn from $p$ (cf. \Cref{defn:signed_sampling}),
  \begin{align*}
    \sup_{\|x\|\leq 1}
    \envert{
      f(x) - \sbr{ c_1 + \frac{\|p\|_{L_1}}{m} \sum_{j=1}^m s_j \sigma'(\ip{\tw_j}{x}) }
    }
    &\leq
    2\omega_f(\delta) + \frac{\|p\|_{L_1}}{\sqrt m}
    \sbr{
      8 \sqrt{ d \ln(m)} + \sqrt{\ln(1/\eta)}
    }.
  \end{align*}
\end{theorem}

\section{Abstract transport mappings, and an RKHS}
\label{sec:rkhs}

\Cref{sec:transport} provided \emph{concrete} transport mappings via Fourier
transforms: it was, for instance, easy to use these constructions to develop
approximation rates for continuous functions.  These constructions
had a major weakness: they were random feature transport mappings, meaning
they arguably did not fully utilize the transport sampling provided in \Cref{sec:samp}.
This section will
develop one abstract approach via RKHSes, but first will revisit the random feature
constructions.

Suppose $f(x) = \int p(\tw) \sigma'(\ip{\tw}{\tx})\dif\tw$ for some density $p$;
as in the proof of \Cref{fact:transport:fourier}, introducing the term $\sfrac{G(\tw)}{G(\tw)}$
gives $f(x) = \int \frac{p(\tw)}{G(\tw)} \sigma'(\ip{\tw}{\tx})\dif G(\tw)$, which is now in the
desired form, however the ratio term can be large (and a truncation is needed to make
it finite in \Cref{fact:transport:fourier}).  The refined construction in
\Cref{fact:transport:fourier:gaussian} achieved a better bound on $\sup_\tw\|\cT(\tw)\|$
by being careful about the scaling of the Gaussian,
and then standardizing it with a change-of-variable transformation, but still it yields a
random feature transport.

Another approach would be to start from the second construction in \Cref{fact:fourier},
which writes $f(x) = \int p(\tw) \sigma(\ip{\tw}{\tx})\dif\tw = \int p(\tw) \ip{\tw}{\tx}
\sigma'(\ip{\tw}{\tx})$, and thus build a transport around $\tw\mapsto p(\tw)\tw$, which
now uses all coordinates.  This transport mapping is still just a rescaling, however, and
does not lead to improvements when plugged into the other parts of this work.

Consider the following approach to building a general $\cT$ and an associated \emph{Reproducing
Kernel Hilbert Space (RKHS)}.
To start, define an inner product $\ip{\cdot}{\cdot}_\cH$ and norm $\|\cdot\|_\cH$ via
$\|\cT\|_\cH^2
= \ip{\cT}{\cT}_{\cH} = \int \|\cT(\tw)\|_2^2\dif G(\tw) = \|\cT\|_{L_2(G)}^2$;
to justify this Hilbert space, note that it gives rise to the usual kernel product
\citep{cho_saul}, namely
\[
  (x,x')\ \mapsto\ {}
  \bbE_{\tw} \Phi(x;\tw)^\T \Phi(x';\tw) = \ip{\Phi(x;\cdot)}{\Phi(x';\cdot)}_\cH,
\]
and moreover our earlier predictors can be written as $\E_{\tw}\ip{\cT(\tw)}{\Phi(x;\tw)}
= \ip{\cT}{\Phi(x;\cdot)}_\cH$.

The utility of these definitions is highlighted in the following bounds;
specifically, while a given $\cT$ may have $\sup_{\tw}\|\cT(\tw)\|_2=\infty$,
truncation can make this quantity finite (and thus \Cref{fact:samp} may be applied),
and the approximation error can be bounded with $\|\cT\|_\cH$.

\begin{proposition}
  \label{fact:rkhs}
  \begin{enumerate}
    \item
      The \emph{output-truncated} transport $\cT_B(\tw)
      := \cT(\tw)\1\sbr{\|\cT(\tw)\|_2\leq B }$
      has approximation error
      \[
        \sup_{\|x\|\leq 1}
        \envert{
          \ip{\cT_B}{\Phi(x;\cdot)}_\cH
          -
          \ip{\cT}{\Phi(x;\cdot)}_\cH
        }
        \leq \frac {\|\cT\|_\cH^2\sqrt{2}}{B^2}.
      \]

    \item
      The \emph{input-truncated} transport
      $\cT_r(\tw) := \cT(\tw)\1[\|\tw\|\leq r]$
      with $r > \sqrt{d}$
      has approximation error
      \[
        \sup_{\|x\|\leq 1}
        \envert{ \ip{\cT_r}{\Phi(x;\cdot)}_\cH - \ip{\cT}{\Phi(x;\cdot)}_\cH }
        \leq
        \frac {\|\cT\|_\cH\sqrt{2}}{e^{(r-\sqrt{d})^2/4}}.
      \]

  \end{enumerate}
\end{proposition}

Unfortunately, this formalism can not be applied to the pre-truncation mapping
$\cT_\infty$ from \Cref{fact:transport:fourier:gaussian}, since
$\|\cT_\infty\|_\cH=\infty$.  Consequently, this approach is left as an interesting
direction for future work.

\section{Open problems}
\label{sec:open}

The main open question is: how much can the rates $\sfrac {1}{\delta^{2d}}$ for ReLU networks
and $\sfrac {1}{\delta^{10d}}$ for their NTK be tightened, and is there a genuine gap?  Expanding
this inquiry,
firstly there are three relevant choices regarding which layers are trained: 
training just the output layer as with random features
\citep{bach_quadrature,bach_convexified}, training just the input layer (as in this work),
and training both layers.
Secondly, for each of these choices, there is a question of norm;
e.g., by requiring the maximum over node
weight Euclidean norms to be small, the NTK regime is enforced.
Are there genuine separations between these settings?
Which settings are most relevant empirically?
What happens beyond the NTK \citep{yuanzhi_beyond_ntk}?

Another direction is to use the Fourier tools of \Cref{sec:samp},
as well as other tools for constructing transportation maps, and identify function classes
with good approximation rates by the NTK and by shallow networks, in particular rates with a merely
polynomial dependence on dimension.

Connecting back to the optimization literature, the referenced NTK optimization
works for the squared loss seem to require a width which scales with $n$, and
the test error sometimes scales with detailed functions of the observed labels,
which require a further argument to go to $0$ (see, e.g., $y^\T (H^\infty)^{-1}
y$ in \citep{2019arXiv190108584A}).  Perhaps such quantities can be replaced
with a function space or other approximation theoretic perspective on the conditional
mean function (and samples thereof)?

Lastly, what are connections to \emph{optimal} transport? It seems natural to choose
$\cT$ as an optimal transport, in which case one would hope the parameter
$B:=\sup_{\tw} \|\cT(\tw)\|_2$ can be small, and moreover easily bounded by the optimal transport
cost, ideally in ways similarly easy to the bounding by the Hilbert norm in \Cref{fact:rkhs}.

\subsection*{Acknowledgements}

The authors are grateful for support from the NSF under grant IIS-1750051,
and from NVIDIA via a GPU grant.

\bibliographystyle{plainnat}
\bibliography{bib}

\appendix

\section{Technical lemmas}
\label{app:technical}

\paragraph{Gaussian Concentration.}
The following \namecref{fact:gauss} collects a few properties of Gaussian concentration needed throughout.

\begin{lemma}
  \label{fact:gauss}
  Let $w\sim G_d$ be a standard Gaussian in $\R^d$,
  and let $r \geq \sqrt{d}$ be given.
  \begin{enumerate}
    \item
      $\displaystyle
      \int \|w\| \dif G(w) \leq \sqrt{d}
      $.
    \item
      $\displaystyle \Pr[ \|w\| > r ] \leq \exp(- (r-\sqrt{d})^2 / 2)$;
      alternatively, with probability at least $1-\eta$,
      $\|w\| \leq \sqrt{d} + \sqrt{2 \ln(1/\eta)}$.
    \item
      $\displaystyle
      \int_{\|w\|> r} \|w\| \dif G(w)
      \leq (r+2) \exp\del{- (r-\sqrt{d})^2/2}
      \leq 2 (\sqrt d + 3)\exp\del{-(r - \sqrt{d})^2/4}
      $.
    \item
      $\displaystyle
      \int_{\|w\|> r} \|w\|^2 \dif G(w)
      \leq 2 (r+2)^2 \exp\del{- (r-\sqrt{d})^2/2}
      \leq 2 (\sqrt d + 7)^2\exp\del{-(r - \sqrt{d})^2/4}
      $.
  \end{enumerate}
\end{lemma}

The more convenient form of some of the inequalities will need to following
technical lemma.

\begin{lemma}
  \label{fact:exp_simp}
  Given $b\geq 0$ and $c>0$ and $a\geq 0$ with $a+b\geq 2c$ and $x\geq b$, then
  \[
    (x+a)\exp(-(x-b)^2/c)
    \leq (a+b)\exp(-(x-b)^2/(2c)),
  \]
  and if moreover $a+b\geq 4c$,
  \[
    (x+a)^2\exp(-(x-b)^2/c)
    \leq (a+b)^2\exp(-(x-b)^2/(2c))
  \]
\end{lemma}
\begin{proof}
  Since $\ln(x+a) \leq \ln(b+a) + (x-b)/(b+a)$,
  \begin{align*}
    (x+a)\exp(-(x-b)^2/c)
    &\leq (a+b)\exp(-(x-b)^2/c + (x-b)/(b+a))
    \\
    &\leq (a+b)\exp(-(x-b)^2/c + (x-b)^2/(2c))
    \\
    &\leq (a+b)\exp(-(x-b)^2/(2c)).
  \end{align*}
  Similarly, multiplying the preceding Taylor expansion by $2$,
  \begin{align*}
    (x+a)^2\exp(-(x-b)^2/c)
    &\leq (a+b)^2\exp(-(x-b)^2/c + 2(x-b)/(b+a))
    \\
    &\leq (a+b)^2\exp(-(x-b)^2/c + 2(x-b)^2/(4c))
    \\
    &\leq (a+b)^2\exp(-(x-b)^2/(2c)).
  \end{align*}
\end{proof}

\begin{proof}[Proof of \Cref{fact:gauss}]
  \begin{enumerate}
    \item
      By Jensen's inequality,
      $\displaystyle
      \int \|w\| \dif G(w)
      \leq \sqrt{ \int \|w\|^2 \dif G(w) }
      = \sqrt{d}$.

    \item
      The claim follows from Gaussian concentration with Lipschitz
      mappings \citep[Theorem 2.4]{2015WainwrightStat210b},
      specifically since $w \mapsto \|w\|$ is $1$-Lipschitz,
      meaning
      \[
        \envert{ \enVert{ w } - \enVert{ w' } }
        \leq
        \enVert{ w - w' },
      \]
      and since $\E \|w\| < \sqrt{d}$.

    \item
      Note that
      \[
        \int_{\|w\|>r} \|w\| \dif G(w)
        \leq
        \sum_{i=0}^\infty
        \int_{r + i < \|w\| \leq r+i+1} (r+i+1) \dif G(w),
      \]
      whereas the Gaussian concentration from the preceding part grants
      \[
        \Pr[ \|w\| > r+i] \leq \exp(- (r+i - \sqrt{d})^2/2)
        \leq \exp(- (r - \sqrt{d})^2/2) \exp(-i^2/2),
      \]
      whereby
      \begin{align*}
        \int_{\|w\|>r} \|w\| \dif G(w)
        &\leq
        (r+1) \int_{\|w\|>r} \dif G(w)
        + \sum_{i=0}^\infty i \int_{\|w\|> r+i} \dif G(w)
        \\
        &\leq
        (r+1) \exp(-(r-\sqrt{d})^2/2)
        + \sum_{i=0}^\infty i \exp(- (r - \sqrt{d})^2/2) \exp(-i^2/2)
        \\
        &\leq
        \exp(-(r-\sqrt{d})^2/2)\sbr{
          r+1 + \sum_{i =0}^\infty i \exp(-i^2/2)
        }
        \\
        &\leq
        \exp(-(r-\sqrt{d})^2/2)\sbr{
          r+2
        }.
      \end{align*}
      The final inequality follows by applying \Cref{fact:exp_simp} with
      $(a,b,c,x) = (3, \sqrt{d}, 2, r)$

    \item
      Proceeding similarly,
      \begin{align*}
        \int_{\|w\|>r} \|w\| \dif G(w)
        &\leq
        \sum_{i=0}^\infty
        \int_{r + i < \|w\| \leq r+i+1} (r+i+1)^2 \dif G(w),
        \\
        &\leq
        2(r+1)^2 \int_{\|w\|>r} \dif G(w)
        + 2\sum_{i=0}^\infty i^2 \int_{\|w\|> r+i} \dif G(w)
        \\
        &\leq
        2(r+1)^2 \exp(-(r-\sqrt{d})^2/2)
        + 2\sum_{i=0}^\infty i^2 \exp(- (r - \sqrt{d})^2/2) \exp(-i^2/2)
        \\
        &\leq
        2\exp(-(r-\sqrt{d})^2/2)\sbr{
          (r+1)^2 + \sum_{i =0}^\infty i^2 \exp(-i^2/2)
        }
        \\
        &\leq
        2\exp(-(r-\sqrt{d})^2/2)\sbr{
          r+2
        }^2.
      \end{align*}
      The final inequality follows by applying \Cref{fact:exp_simp} with
      $(a,b,c,x) = (7, \sqrt{d}, 2, r)$.

  \end{enumerate}
\end{proof}

\paragraph{Fourier transforms.}
The convention for the Fourier transform used here is
\[
  \hat f(w) = \int f(x) \exp(2\pi i w^\T x) \dif x;
\]
see for instance \citep[Section 8.8]{folland} for a discussion of other conventions,
and the resulting tradeoffs.  Note also the polar decomposition notation
$\hat f(w) = |\hat f(w)| \exp(2\pi i \theta_f(w))$ with $|\theta_f(w)|\leq 1$.
The following \namecref{fact:fourier:helper} collects a few properties used throughout.

\begin{lemma}
  \label{fact:fourier:helper}
  \begin{enumerate}
    \item
      $|\hat f| \leq \|f\|_{L_1}$.
    \item
      $\widehat{f * g} = \hat f \hat g$
      and $|\widehat{f * g}| \leq \|f\|_{L_1} |\hat g|$.
    \item
      Let $\alpha > 0$ be given and define $\phi:= (2\pi\alpha)^{-1}$.
      Then $|\hat G_\alpha | = \hat G_\alpha$ (meaning $\hat G_\alpha$ has no radial component, thus $\theta_{G_\alpha}(w) = 0$),
      and
      \[
        \hat G_\alpha(w)
        = (2\pi\alpha^2)^{-d/2} G_\phi(w)
        = (2\pi\phi^2)^{d/2} G_\phi(w)
        = (2\pi)^{d/2} G(w/\phi).
      \]
  \end{enumerate}
\end{lemma}
\begin{proof}
  \begin{enumerate}
    \item
      Directly,
      \[
        |\hat h (w)|
        \leq
        \int |h(x) | \cdot |\exp(2\pi i w^\T x)| \dif x
        \leq
        \|h\|_{L_1},
      \]

    \item
      The first equality is standard \citep[Theorem 8.22c]{folland},
      and the inequality combines it with the preceding bound.

    \item
      The form of $\hat G_\alpha$ and the first displayed inequality
      are standard \citep[Proposition 8.24]{folland}.
      The second and third inequalities use the
      choice of $\phi$ and the form of $G_\phi$.
  \end{enumerate}
\end{proof}

\paragraph{ReLU representation.}
Lastly, the exact ReLU representation constructions (e.g., \Cref{fact:fourier})
will use the following folklore lemma to write a univariate twice continuously differentiable
function as an infinite width ReLU network.

\begin{lemma}
  \label{fact:relu:uni}
  Let $f :\R\to\R$ be given with $f(0) = f'(0) = 0$ and $f''$ continuous.
  For any $z\geq 0$,
  \[
    f(z) = \int_0^\infty \sigma(z-b) f''(b) \dif b.
  \]
\end{lemma}
\begin{proof}Using integration by parts,
  \begin{align*}
    \int_0^\infty \sigma(z-b) f''(b) \dif b
    &=
    \int_0^z (z-b) f''(b) \dif b
    \\
    &=
    z\int_0^z f''(b) \dif b
    -
    \int_0^z b f''(b) \dif b
    \\
    &=
    z f'(b) |_0^z
    - \del{ b f'(b)|_0^z - \int_0^z f'(b) \dif b }
    \\
    &=
    z f'(z) - z f'(0)
    - \del{ z f'(z) - 0 - f(z) + f(0) }
    \\
    &= f(z).
  \end{align*}
\end{proof}

\section{Sampling tools: Maurey's lemma and co-VC dimension}
\label{app:sampling_tools}

This section collects various sampling tools used as a basis for \Cref{sec:samp}.
First is a proof of \Cref{fact:maurey}, which here is combined with an application of
McDiarmid's inequality to give a high probability guarantee.

\begin{proof}[Proof of \Cref{fact:maurey}]
  Following the usual Maurey scheme
  \citep{pisier1980remarques},
  \begin{align*}
    \E_{(v_j)_{j=1}^m}
    \enVert{ f - \frac 1 m\sum_j g_j }_{L_2(P)}^2
    &=
    \frac 1 {m^2}
    \E_{(v_j)_{j=1}^m}
    \enVert{\sum_j\del{ f - g_j } }_{L_2(P)}^2
    \\
    &=
    \frac 1 {m^2}
    \E_{(v_j)_{j=1}^m}
    \sum_j
    \enVert{ f - g_j }_{L_2(P)}^2
    \\
    &=
    \frac 1 {m}
    \E_{v_1}
    \enVert{ f - g_1 }_{L_2(P)}^2
    \\
    &=
    \frac 1 {m}
    \E_{v_1}
    \del{
      \enVert{ g_1 }_{L_2(P)}^2
      - \enVert{ f }_{L_2(P)}^2
    }
    \\
    &\leq
    \frac 1 {m}
    \bbE_{v_1}
    \enVert{ g_1 }_{L_2(P)}^2
    \\
    &\leq
    \frac 1 {m}
    \sup_{v\in\cV}
    \enVert{ g(\cdot;v) }_{L_2(P)}^2.
  \end{align*}
  The high probability bound will follow from McDiarmid's inequality.
  To establish the bounded differences property, 
  define
  \[
    F(V) := F( (v_1,\ldots, v_m) ) = \enVert{ f - \frac 1 m \sum_j g(\cdot;v_j) }_{L_2(P)},
  \]
  and note from the general metric space inequality $\envert{ \|p\| - \|q\| } \leq \|p-q\|$ that
  for any $V = (v_1,\ldots,v_m)$ and $V' = (v'_1,\ldots,v'_m)$ differing only on
  a single $v_k\neq v'_k$,
  \begin{align*}
    \enVert{ F(V) - F(V') }
    &\leq
    \enVert{ \frac 1 m \sum_j g(\cdot;v_j) - \frac 1 m \sum_j g(\cdot;v'_j) }_{L_2(P)}
    \\
    &=
    \frac 1 m 
    \enVert{ g(\cdot;v_k) - g(\cdot;v'_k) }_{L_2(P)}
    \\
    &\leq
    \frac 1 m 
    \del{
      \enVert{ g(\cdot;v_k) }_{L_2(P)}
      +
      \enVert{ g(\cdot;v'_k) }_{L_2(P)}
    }
    \\
    &\leq
    \frac 2 m \sup_{v\in\cV}
    \enVert{ g(\cdot;v) }_{L_2(P)}.
  \end{align*}
  Thus, with probability at least $1-\eta$, McDiarmid's inequality grants,
  \[
    F(V)
    \leq \bbE_V F(V) + \sup_{v\in\cV} \|g(\cdot;v)\|_{L_2(P)} \sqrt{\frac{2\ln(1/\eta)}{m}},
  \]
  and the statement follows by Jensen's inequality, specifically
  $\bbE_V F(V) \leq \sqrt{ \bbE_V F(V)^2}$.
\end{proof}

Maurey's lemma also applies to sampling from signed densities in the sense
of \Cref{defn:signed_sampling}.

\begin{lemma}
  \label{fact:signed_maurey}
  Let $f(x) = \int p(\tw) g(\ip{\tw}{\tx}) \dif \tw$
  be given with $\|p\|_{L_1}<\infty$ and $p$ is supported on a ball of radius $B$,
  and let $((s_j,\tw_j))_{j=1}^m$ be sampled from $p$ as in \Cref{defn:signed_sampling},
  and define $g_j(x) := g(\ip{\tw}{\tx})$.
  With probability at least $1-\eta$,
  \begin{align*}
    \enVert{ f - \frac {\|p\|_{L_1}} m \sum_{j=1}^m s_j g_j }_{L^2(P)}
    &\leq
    \sup_{\|\tw\|\leq B} \|g(\ip{\tw}{\cdot})\|_{L_2(P)} \|p\|_{L_1}
    \sbr{ \frac{ 1 + \sqrt{2\ln(1/\eta)}}{\sqrt {m}} }.
  \end{align*}
\end{lemma}

\begin{proof}Since
  \[
    \int p(\tw) g(\ip{\tw}{\tx})\dif\tw
    =
    \|p\|_{L_1} \int \sgn(p) \frac {|p(\tw)|}{\|p\|_{L_1}} g(\ip{\tw}{\tx})\dif\tw
    =
    \|p\|_{L_1} \E_{s,\tw} s g(\ip{\tw}{\tx}),
  \]
  the sampling procedure indeed provides an unbiased estimate of the integral,
  and thus by Maurey's Lemma (cf. \Cref{fact:maurey}),
  with probability at least $1-\eta$,
  \begin{align*}
    \enVert{ f - \frac {\|p\|_{L_1}} m \sum_{j=1}^m s_j g_j }_{L^2(P)}
    &=
    \|p\|_{L_1}
    \enVert{ \bbE_{s,\tw} s g(\ip{\tw}{\cdot}) - \frac {1} m \sum_{j=1}^m s_j g_j }_{L^2(P)}
    \\
    &\leq
    \|p\|_{L_1} \sup_{\substack{s\in\{\pm 1\}\\\|\tw\|\leq B}} \|sg(\ip{\tw}{\cdot})\|_{L_2(P)}
    \sbr{ \frac{ 1 + \sqrt{2\ln(1/\eta)}}{\sqrt {m}} }
    \\
    &=
    \|p\|_{L_1} \sup_{\|\tw\|\leq B} \|g(\ip{\tw}{\cdot})\|_{L_2(P)}
    \sbr{ \frac{ 1 + \sqrt{2\ln(1/\eta)}}{\sqrt {m}} }.
  \end{align*}
\end{proof}

Lastly, here is a uniform norm analog of the preceding $L_2(P)$ signed density sampling bound.
Interestingly, the bound only gives a $\sqrt{d}$ degradation with $\sigma'$, and no degradation
for $\sigma$.  The method of proof is to use uniform convergence, but with data and parameters
switched; consequently, this has been called ``co-VC dimension''
\citep{gurvits_koiran,tewari_sun_gilbert_rkhs}.
The proof is somewhat more complicated than the proof of the Maurey lemma, and in particular
needs to be a bit more attentive to the fine-grained structure of the functions being sampled.

\begin{lemma}
  \label{fact:sample:uniform}
  Let density $p:\R^{d+1}\to \R$ with $\|p\|_{L_1} < \infty$,
  and let $((s_j,w_j))_{j=1}^m$ be a sample from $p$ in the sense of \Cref{defn:signed_sampling}.

  \begin{enumerate}
    \item
      With probability at least $1-2\eta$,
      \[
        \sup_{\|x\|\leq 1}
        \envert{ \int p(\tw) \sigma'(\ip{\tw}{\tx})\dif \tw
          - \frac {\|p\|_{L_1}}{m} \sum_j s_j \sigma'(\ip{\tw_j}{\tx})
        }
        \leq
        \frac {\|p\|_{L_1}}{\sqrt m}
        \sbr{
          \sqrt{8 (d+1) \ln(m+1)} + \sqrt{\ln(1/\eta)}
        }.
      \]

    \item
      Suppose $p$ is supported on the set $\cW := \{ \tw \in \R^{d+1} : \|w\|\leq r, |b| \leq \|w\|\}$.
      With probability at least $1-2\eta$,
      \[
        \sup_{\|x\|\leq 1}
        \envert{ \int p(\tw) \sigma(\ip{\tw}{\tx})\dif \tw
          - \frac {\|p\|_{L_1}}{m} \sum_j s_j \sigma(\ip{\tw_j}{\tx})
        }
        \leq
        \frac {4r \|p\|_{L_1}}{\sqrt m}
        \sbr{
          1 + \sqrt{\ln(1/\eta)}
        }.
      \]
  \end{enumerate}
\end{lemma}
\begin{proof}[Proof of \Cref{fact:sample:uniform}]
  In both cases, letting $g$ denote either of $\sigma'$ or $\sigma$,
  \begin{align*}
    &\hspace{-2em}
    \sup_{\|x\|\leq 1}
    \envert{ \int p(\tw) g(\ip{\tw}{\tx})\dif \tw
      - \frac {\|p\|_{L_1}}{m} \sum_j s_j g(\ip{\tw_j}{\tx})
    }
    \\
    &=
    \|p\|_{L_1}
    \sup_{\|x\|\leq 1}
    \envert{ \E_p s g(\ip{\tw}{\tx})\dif \tw
      - \frac {1}{m} \sum_j s_j g(\ip{\tw_j}{\tx})
    },
  \end{align*}
  and at this point it is a classical uniform deviations problem,
  but with the role of parameter and data swapped, an approach which has
  been used before (sometimes under the heading ``co-VC dimension''
  \citep{gurvits_koiran,tewari_sun_gilbert_rkhs}).
  Continuing, with probability at least $1-2\eta$, standard Rademacher complexity
  \citep{shai_shai_book} grants
  \begin{align*}
    \sup_{\|x\|\leq 1}
    \envert{ \E_p s g(\ip{\tw}{\tx})\dif \tw
      - \frac {1}{m} \sum_j s_j g(\ip{\tw_j}{\tx})
    }
    &\leq
    2\Rad\del{\cbr{ ( s_j g(\ip{\tw_j}{\tx}) )_{j=1}^m : \|x\| \leq 1 }}
        \\
    &\quad + 
    3 \sup_{\substack{\tw\in\cW\\\|x\|\leq 1}} | g(\ip{\tw}{\tx}) | \sqrt{\frac{\ln(1/\eta)}{2m}}.
  \end{align*}
  where $\cW$ is a constraint set on $\tw$ (when $g=\sigma'$, it is $\R^{d+1}$,
  whereas with $g=\sigma$ it is $|b| \leq \|w\|\leq r$).
  To simplify further, note that a Rademacher random vector
  $(\eps_1,\ldots,\eps_m)$ is distributionally equivalent to
  $(s_1\eps_1,\ldots,s_m\eps_m)$ for any fixed vector of signs $(s_1,\ldots,s_m)$,
  and therefore
  \begin{align*}
    \Rad\del{\cbr{ ( s_j g(\ip{\tw_j}{\tx}) )_{j=1}^m : \|x\| \leq 1 }}
    &=
    \frac 1 n \bbE_\eps \sup_{\|x\|\leq 1} \sum_j s_j \eps_j g(\ip{\tw}{\tx})
    \\
    &=
    \frac 1 n \bbE_\eps \sup_{\|x\|\leq 1} \sum_j \eps_j g(\ip{\tw}{\tx})
    \\
    &=
    \Rad\del{\cbr{ ( g(\ip{\tw_j}{\tx}) )_{j=1}^m : \|x\| \leq 1 }}.
  \end{align*}
  Combining these steps, with probability at least $1-2\eta$,
  \begin{align*}
    &\hspace{-2em}
    \sup_{\|x\|\leq 1}
    \envert{ \int p(\tw) g(\ip{\tw}{\tx})\dif \tw
      - \frac {\|p\|_{L_1}}{m} \sum_j s_j g(\ip{\tw_j}{\tx})
    }
    \\
    &\leq
    \|p\|_{L_1} \sbr{
    2\Rad\del{\cbr{ ( g(\ip{\tw_j}{\tx}) )_{j=1}^m : \|x\| \leq 1 }}
    + 
    3 \sup_{\substack{\tw\in\cW\\\|x\|\leq 1}} | g(\ip{\tw}{\tx}) | \sqrt{\frac{\ln(1/\eta)}{2m}}
  }.
  \end{align*}
  The proof now splits into two cases $g\in \{\sigma', \sigma\}$, bounding the remaining terms.

  \begin{enumerate}
    \item
      Since the range of $\sigma'$ is $\{0,1\}$,
      \[
        \sup_{\substack{\tw\in\cW\\\|x\|\leq 1}} | \sigma'(\ip{\tw}{\tx}) | \leq 1,
      \]
      and the Rademacher complexity is the VC dimension
      of linear predictors, thus
      \[
        \Rad\del{\cbr{ ( \sigma'(\ip{\tw_j}{\tx}) )_{j=1}^m : \|x\| \leq 1 }}
        \leq \sqrt{\frac {2(d+1)\ln( m + 1 )}{m}}.
      \]

    \item
      In the case $g = \sigma$, since $\cW := \{ \tw \in \R^{d+1} : \|w\| \leq r, |b| \leq \|w\|\}$,
      \[
        \sup_{\substack{\tw\in\cW\\\|x\|\leq 1}} | \sigma(\ip{\tw}{\tx}) |
        \leq
        \sup_{\substack{\|w\|\leq r\\|b|\leq \|w\|\\\|x\|\leq 1}} | w^\T x + b |
        \leq
        2r.
      \]
      Moreover, the Rademacher complexity is a standard combination of the Lipschitz composition
      rule and linear prediction rules \citep{shai_shai_book}, and thus
      \[
        \Rad\del{\cbr{ ( \sigma'(\ip{\tw_j}{\tx}) )_{j=1}^m : \|x\| \leq 1 }}
        \leq 
        \Rad\del{\cbr{ ( \ip{\tw_j}{\tx} )_{j=1}^m : \|x\| \leq 1 }}
        \leq
        \frac {4r}{\sqrt m}.
      \]
  \end{enumerate}
\end{proof}

\section{Deferred proofs from \Cref{sec:samp}}

For convenience throughout this appendix, define
\[
  B:= \sup_{\tw} \|\cT(\tw)\|
  \qquad\textup{and}\qquad
  B_\eps := \sup_{\tw,s} \|\cT_\eps(\tw,s)\|_2 \leq \frac {B}{\eps \sqrt m} + R.
\]
The first step is to prove \cref{eq:samp:1}, restated here as follows.

\begin{lemma}
  \label{fact:samp:1}
  With probability at least $1-\eta$ over $(\tW,S)$,
  \[
    \enVert{\sum_{j=1}^m \ip{\tau_j}{\phi_j(\cdot)}
      - \bbE_{\tw} \ip{\cT(\tw)}{\Phi(\cdot;\tw)}
    }_{L^2(P)}
    \leq
    \eps B_\eps
    \sbr{ \sqrt{2} + 2 \sqrt{\ln(1/\eta)} }.
  \]
\end{lemma}

\begin{proof}The proof proceeds by applying Maurey sampling (cf. \Cref{fact:maurey})
  to the functions $g_j(x) := m \ip{\tau_j}{ \phi_j(x)}$,
  noting by \Cref{fact:sampling} that
  \begin{align*}
    f(x)
    := \E_{\tW,S} \frac 1 m \sum_j g_j(x)
    = \E_{\tW, S} \sum_j \ip{\tau_j}{\phi_j(x)}
    = \E_{\tW, S} \ip{\cT_\eps(\tW,S)}{ \Phi_\eps(x; \tW, S) }
    = \E_{\tw} \ip{ \cT(\tw) }{ \Phi(x;\tw) }.
  \end{align*}
  Applying \Cref{fact:maurey}, with probability at least $1-\eta$,
  \[
    \enVert{ f
      - \frac 1 m \sum_{j=1}^m g_j
    }_{L^2(P)}
    \leq
    \frac {\sup_{\tw, s} m \|\ip{\cT_\eps(\tw,s)}{\Phi_\eps(\cdot;\tw,s)}\|_{L_2(P)}}{\sqrt m}
    \sbr{ 1 + \sqrt{2\ln(1/\eta)} },
  \]
  where
  \[
    \sup_{\tw, s} \|\ip{\cT_\eps(\tw,s)}{\Phi_\eps(\cdot;\tw,s)}\|_{L_2(P)}^2
    \leq
    \sup_{\tw, s} \bbE_x \enVert{\cT_\eps(\tw)}_2^2 \enVert{\Phi_\eps(x;\tw,s)}_2^2
    \leq
    \frac {2\eps^2B_\eps^2}{m}.
  \]
\end{proof}

Next, the restatement of \cref{eq:samp:2} is as follows.

\begin{lemma}
  \label{fact:samp:2}
  With probability at least $1-\eta$,
  If $R \geq \sqrt{d} + 2\sqrt{ \ln\del{ \frac {\eps\sqrt{m\pi}}{B \sqrt{2}}} }$,
  then with probability at least $1-\eta$,
  \[
    \enVert{
      \sum_j \ip{\tau_j}{\phi_j(\cdot)}
      - \sum_j \frac{s_j \eps}{\sqrt m} \sigma(\ip{\tau_j}{\tx})
    }_{L_2(P)}
    \leq
    \frac {B_\eps B}{m\sqrt{\pi}}
    + \eps B_\eps
    \sbr{
      \sqrt{2}  + 2 \sqrt{\ln(1/\eta)}
    }.
  \]
\end{lemma}

Recall that the proof of \Cref{fact:samp:2}, as discussed in the body, must
calculate the fraction of activations which change, which was collected into
\Cref{fact:activations}.

\begin{proof}[Proof of \Cref{fact:activations}]
  Consider an idealized $\cT_\eps'$ which does not truncate, whereby
  \begin{align*}
    \envert{ |\ip{\cT_\eps'(\tw)}{\tx}| - |\ip{\tw}{\tx}| }
    &\leq
    \envert{ 
      \frac {\ip{\cT(\tw)}{\tx}}{\eps\sqrt{m}}
+ \ip{\tw}{\tx} - \ip{\tw}{\tx}
    }
\leq
    \frac {B\|\tx\|}{\eps\sqrt{m}}
.
  \end{align*}
  The event $\sbr{ \sgn(\ip{\tw}{\tx}) \neq \sgn(\ip{\cT_\eps'(\tw)}{\tx}) }$
  implies the event $\sbr{ |\ip{\tw}{\tx}| \leq \sfrac {B\|\tx\|}{\eps\sqrt{m}} }$,
  and thus, additionally using rotational invariance of the Gaussian,
  \begin{align*}
    \E_{\tw}
    \envert{ \sigma'(\ip{\tw}{\tx}) - \sigma'(\ip{\cT_\eps'(\tw)}{\tx}) }
    &= \Pr_{\tw}\sbr{ \sgn(\ip{\tw}{\tx}) \neq \sgn(\ip{\cT_\eps'(\tw)}{\tx}) }
    \\
    &\leq
    \Pr_{\tw}\sbr{ |\ip{\tw}{\tx}| \leq \sfrac {B\|\tx\|}{\eps\sqrt{m}} }
    \\
    &=
    \Pr_{\tw}\sbr{ |w_1|\cdot \|\tx\| \leq \sfrac {B\|\tx\|}{\eps\sqrt{m}} }
    \\
    &=
    \frac 1 {\sqrt{2\pi}} \int_{-B / (\eps\sqrt{m})}^{B / (\eps\sqrt{m})}
    e^{-z^2/2}\dif z
    \\
    &\leq
    \frac{B}{\eps}\sqrt{\frac{2}{m\pi}}.
  \end{align*}
  Returning to the general case with truncation, by \Cref{fact:gauss},
  using the assumed lower bound on $R$,
  \[
    \Pr[ \|\tw\| > R ]
    \leq \exp(- (R - \sqrt{d})^2 /2 ) \leq \frac {B}{\eps}\sqrt{\frac {2}{m\pi}},
  \]
  which gives the final bound via triangle inequality.
\end{proof}

With \Cref{fact:activations} in hand, the proof of \Cref{fact:samp:2} is now an application
of Maurey's lemma, with an invocation of positive homogeneity to massage terms.

\begin{proof}[Proof of \Cref{fact:samp:2}]
  The approach is once again to apply Maurey sampling (cf \Cref{fact:maurey}).
  To this end, define
  \[
    g(x;\tw,s) := m \del{
      \ip{\cT_\eps(\tw,s)}{\Phi_\eps(x;\tw,s)} - 
      \frac{s\eps}{\sqrt m}\sigma\del{\ip{\cT_\eps(\tw,s)}{\tx}}
    }
    \quad\textup{and}\quad
    f(x) = \E_{\tw,s} g(x;\tw,s),
  \]
  as well as $g_j(x) := g(x;\tw_j,s_j)$.
  Using this notation, the goal of this proof is to upper bound
  \[
    \enVert{
      \sum_j \ip{\tau_j}{\phi_j(\cdot)}
      - \sum_j \frac{s_j\eps}{\sqrt m} \sigma(\ip{\tau_j}{\cdot})
    }_{L_2(P)}
    =
    \enVert{ \frac 1 m \sum_j g_j }_{L_2(P)}.
  \]
  By \Cref{fact:maurey}, with probability at least $1-\eta$,
  \begin{align*}
\enVert{ \frac 1 m \sum_j g_j }_{L_2(P)}
    &\leq
    \|f\|_{L_2(P)}
    +
    \enVert{f -  \frac 1 m \sum_j g_j }_{L_2(P)}
    \\
    &\leq
    \|f\|_{L_2(P)}
    +
    \sup_{\tw,s} \|g(\cdot;\tw,s)\|_{L_2(P)}
    \sbr{ \frac { 1 + \sqrt{2\ln(1/\eta)} }{\sqrt{m}} }.
  \end{align*}
  To control these terms,
  fixing any $(\tw,s)$,
  it holds by positive homogeneity of $\sigma$ that
  \begin{align*}
    \enVert{ g(x;\tw,s)}_{L_2(P)}^2
    &=
    m^2 \E_x \ip{\cT_\eps(\tw,s)}{\Phi_\eps(x;\tw,s)
      - \frac{s\eps}{\sqrt m} \Phi(x;\cT_\eps(\tw,s))}^2
    \leq
    m^2 B_\eps^2 \E_x \frac {\eps^2\|x\|^2}{m}
    \leq 2 m \eps^2 B_\eps^2.
  \end{align*}
  On the other hand, by \Cref{fact:activations}, for any $\|x\|\leq 1$,
  \begin{align*}
    |f(x)|
    &\leq \bbE_{\tw,s} \envert{ \ip{\cT_\eps(\tw,s)}{\Phi_\eps(x;\tw,s)
    - \frac{s \eps}{\sqrt m} \Phi(x;\cT_\eps(\tw,s))} }
    \\
    &\leq
    B_\eps
    \bbE_{\tw,s} \frac {\eps 
      \|\tx\| \envert{ \sigma'(\cT_\eps(\tw)^\T \tx) - \sigma'(\tw^\T\tx) }
    }{\sqrt m}
    \\
    &\leq
    \frac {\eps B_\eps}{\sqrt m} \del{ \frac B {\eps \sqrt{ m \pi} }},
  \end{align*}
  which also upper bounds $\|f\|_{L_2(P)}$.
\end{proof}

The proof of \Cref{fact:samp} now follows by combining \Cref{fact:samp:1,fact:samp:2}.

\begin{proof}[Proof of \Cref{fact:samp}]
  By \Cref{fact:gauss} and a union bound on $(\tw_1,\ldots,\tw_m)$,
  $\max_j\|\tw_j\|\leq R$, thus
  \[
    \max_j \|\cT_\eps(\tw_j) - \tw_j\|
    \leq
    \max_j \frac{ \|\cT(\tw_j)\| }{ \eps \sqrt{m} }
    \leq
    \frac{ B }{ \eps \sqrt{m} }.
  \]
  Moreover, $R \geq \sqrt{d} + 2\sqrt{ \ln\del{ \frac {\eps\sqrt{m\pi}}{B \sqrt{2}}} }$,
  and thus the two other bounds are from \Cref{fact:samp:1,fact:samp:2}.
\end{proof}

\section{Deferred proofs from \Cref{sec:transport}}

The first core lemma shows how to write a target function $f$ as an infinite-width network
via its Fourier transform.

\begin{proof}[Proof of \Cref{fact:fourier}]
  The first steps are the same for $\sigma$ and $\sigma'$,
  and indeed match the initial steps of \citep{barron_nn},
  namely
  \begin{align*}
    f(x) - f(0)
    &= \RE \int \exp(2\pi i x^\T w) \hat f(w) \dif w
    \\
    &= \RE \int \exp(2\pi i x^\T w + 2\pi i \theta_f(w) ) | \hat f(w) | \dif w
    \\
    &= \int \cos\del{ 2\pi(x^\T w + \theta_f(w) ) } | \hat f(w) | \dif w.
  \end{align*}
  For convenience, define $h(z) := \cos(2\pi z)$, whereby
  \begin{equation}
    f(x) - f(0)
    = \int h(x^\T w + \theta_f(w) ) | \hat f(w) | \dif w,
    \label{eq:barron}
  \end{equation}
  and the proofs not differ for both activations and from \citep{barron_nn}.
  \begin{enumerate}
    \item Consider first $\sigma'$.
      Since $\|x\|\leq 1$, by Cauchy-Schwarz it suffices to approximate $h$
      along the interval $[-\|w\| + \theta_f(w), \|w\| + \theta_f(w)]$.
      By the fundamental theorem of calculus,
      \begin{align*}
        &\hspace{-2em}h(\ip{w}{x} + \theta_f(w)) - h\del{-\|w\| + \theta_f(w)}
        \\
        &= \int_{-\|w\| + \theta_f(w)}^{\ip{w}{x}+\theta_f(w)} h'(b) \dif b
        \\
        &= \int h'(b) \1[ \ip{w}{x} + \theta_f(w) \geq b ] \1 [b \geq -\|w\| + \theta_f(w)]\dif b
\\
        &= -\int h'(\theta_f(w)-b) \1[ x^\T w  + b \geq 0] \1 [\|w\| \geq b ]\dif b,
        &b\mapsto \theta_f(w) - b
        \\
&= -\int h'(\theta_f(w)-b) \1[ x^\T w  + b \geq 0] \1 [\|w\| \geq |b| ]\dif b,
      \end{align*}
      where the last step follows since $\1[x^\T w + b \geq 0]$ implies $b \geq - \|w\|$.
      Plugging this back in to \cref{eq:barron}
      and still using $h(z) = \cos(2\pi z)$,
      \begin{align*}
        f(x) - f(0)
        &=
        \int |\hat f(w)| \cos\del{2\pi(x^\T w + \theta_f(w))} \dif w
        \\
        &=
        \int |\hat f(w)|
        \sbr{ h(-\|w\| + \theta_f(w)) - \int h'(\theta_f(w) - b)\1[x^\T w + b\geq 0] \1[\|w\|\geq |b|] \dif b }
        \dif w,
      \end{align*}
      which after pushing more terms onto the left hand side gives
      \begin{align*}
        &\hspace{-2em}
        f(x) - f(0) - \int |\hat f(w) | h(\theta_f(w) - \|w\|)\dif w
        \\
        &=
        - \iint |\hat f(w)|
        h'(\theta_f(w) - b)\1[x^\T w + b\geq 0] \1[\|w\|\geq |b|]
        \dif b \dif w
        \\
        &=
        2\pi \int |\hat f(w)|
        \sin(2\pi(\theta_f(w) - b)) \sigma'(\ip{\tw}{\tx}) \1[\|w\|\geq |b|]
        \dif \tw,
      \end{align*}
      which gives $F_\infty = f$ for $\|x\|\leq 1$.
      To bound the error of $F_r$, note by the form of $F_\infty$ for any $\|x\|\leq 1$ that
      \begin{align*}
        \envert{ f(x) - F_r(x) }
        &=
        \envert{
          2\pi \int_{\|w\|>r} \int |\hat f(w)|
          \sin(2\pi(\theta_f(w) - b)) \sigma'(\ip{\tw}{\tx}) \1[\|w\|\geq |b|]
          \dif b \dif w
        }
        \\
        &\leq
        2\pi \int_{\|w\|>r} \int_{|b|\leq \|w\|} |\hat f(w)|
        |\sin(2\pi(\theta_f(w) - b))|  \sigma'(\ip{\tw}{\tx})
        \dif b \dif w
        \\
        &\leq
        2\pi \int_{\|w\|>r} |\hat f(w)| \int_{|b|\leq \|w\|}
        \dif b \dif w
        \\
        &=
        4\pi \int_{\|w\|>r} \|w\|\cdot |\hat f(w)| \dif w.
      \end{align*}

    \item
      Now consider $\sigma$.
      Rather than using FTC as above, this proof replaces $h$ with ReLUs
      via \Cref{fact:relu:uni}, which requires a function which is both zero and flat at 0.
      To this end, define
      \[
        H(b) = h(b + q) - \del{ h(q) + b h'(q) }
        \quad\textup{with}\ {}
        q = - \|w\| + \theta_f(w),
      \]
      whereby $H(0) = 0 = H'(0)$.
      Invoking \Cref{fact:relu:uni} on $H$ gives,
      for any $z := w^\T x + \theta_f(w) \geq q$,
      \begin{align*}
        h(z) - \del{ h(q) + (z-q) h'(q) }
        &= H(z - q)
        \\
        &= \int  H''(b) \sigma(z - q - b) \1[b \geq 0] \dif b
        \\
        &= \int  H''(b) \sigma(w^\T x + \theta_f(w) + \|w\| - \theta_f(w) - b) \1[b \geq 0] \dif b
        \\
        &= - \int  H''(\|w\| - b) \sigma(w^\T x + b) \1[\|w\| \geq b] \dif b
        &b \mapsto \|w\| - b
        \\
        &= - \int  H''(\|w\| - b) \sigma(w^\T x + b) \1[\|w\| \geq |b|] \dif b,
      \end{align*}
      the final equality since $-b > \|w\|$ implies $w^\T x + b \leq \|w\| + b < 0$,
      thus $\sigma(\tw^\T \tx) = 0$ and this case has no effect.
      Plugging this back into \cref{eq:barron},
      \begin{align*}
        f(x) - f(0)
        &=
          \int |\hat f(w) | h(z) \dif w
        \\
        &=
        \int |\hat f(w) | \sbr{
          h(q) + (z-q)h'(q)
          - \int H''(\|w\| - b) \sigma(\tw^\T \tx) \1[\|w\|\geq |b|] \dif b
        }\dif w
        \\
        &=
        \int |\hat f(w) | \sbr{
          h(q) + (w^\T x + \|w\|)h'(q)
          - \int H''(\|w\| - b) \sigma(\tw^\T \tx) \1[\|w\|\geq |b|] \dif b
        }\dif w,
      \end{align*}
      which gives $Q_\infty = f$ for $\|x\|\leq 1$ after expanding $h$ and $H$.
      To bound the error of
      $Q_r$, for any $\|x\|\leq 1$
      \begin{align*}
        \envert{ f(x) - Q_r(x) }
        &= \envert{
        \int_{\|w\|>r} |\hat f(w) |
          \int H''(\|w\| - b) \sigma(\tw^\T \tx) \1[\|w\|\geq |b|] \dif b
        }\dif w
        \\
        &\leq
\int_{\|w\|> r} |\hat f(w) |
          \int | H''(\|w\| - b)| \sigma(\tw^\T \tx) \1[\|w\|\geq |b|] \dif b
        \dif w
        \\
        &\leq
        4\pi^2 
        \int_{\|w\|>r} |\hat f(w) |
          \int_{-\|w\|}^{\|w\|}  \sigma(\tw^\T \tx) \dif b
        \dif w
        \\
        &\leq
        4\pi^2 
        \int_{\|w\|>r} |\hat f(w) |
          \int_{-\|w\|}^{\|w\|} ( \|w\| + |b| ) \dif b
        \dif w
        \\
        &\leq
        12 \pi^2 
        \int_{\|w\|>r} \|w\|^2 \cdot |\hat f(w) | \dif w.
      \end{align*}

  \end{enumerate}
\end{proof}

Next, \Cref{fact:transport:fourier} converts \Cref{fact:fourier} into a (random feature)
transport map by introducing the fraction $\sfrac {G(\tw)}{G(\tw)}$.

\begin{proof}[Proof of \Cref{fact:transport:fourier}]
  Starting from the construction in \Cref{fact:fourier},
  again using $h(z) = \cos(2\pi z)$ for convenience,
  and manually introducing a factor $G(\tw)$,
  \begin{align*}
    &\hspace{-2em}
    f(x) - f(0) - \int |\hat f(w) | h(\theta_f(w) - \|w\|)\dif w
    \\
    &=
    - \iint |\hat f(w)|
    h'(\theta_f(w) - b)\1[x^\T w + b\geq 0] \1[\|w\|\geq |b|]
    \dif b \dif w
    \\
    &=
    - \int \frac{|\hat f(w)|}{G(\tw)}
    h'(\theta_f(w) - b) \sigma'(\ip{\tw}{\tx}) \1[\|w\|\geq |b|]
    \dif G(\tw).
  \end{align*}
  To construct $\cT_\infty$,
  rotational invariance of the Gaussian
  gives $\bbE \1[\tw^\T \tx \geq 0] = \sfrac 1 2$,
  thus
  \[
    f(0) + \int |\hat f(w) | h(\theta_f(w) - \|w\|)\dif w
    =
    \int 2 \sbr{ 
      f(0) + \int |\hat f(v) | h(\theta_f(v) - \|v\|)\dif v
    } \sigma'(\ip{\tw}{\tx}) \dif G(\tw),
  \]
  and transport mapping is $\cT_\infty(w,b) = (0, p_\infty(\tw)) \in \R^d\times \R$ with
  \begin{align*}
    p_\infty(\tw)
    &=
    2 \sbr{ 
      f(0) + \int |\hat f(v) | h(\theta_f(v) - \|v\|)\dif v
    }
    - \frac{|\hat f(w)|}{G(\tw)}
    h'(\theta_f(w) - b) \1[\|w\|\geq |b|],
  \end{align*}
  By construction, $\bbE_{\tw} \ip{\cT_r(\tw)}{\Phi(x;\tw)} = F_r(x)$,
  and therefore \Cref{fact:fourier} grants for all $\|x\|\leq 1$
  \[
    f(x) = \bbE_{\tw} \ip{\cT_\infty(\tw)}{\Phi(x;\tw)}
    \quad\textup{and}\quad
    \envert{ f(x) - \bbE \ip{\cT_r(\tw)}{\Phi(x;\tw)} }
    \leq
    4\pi \int_{\|w\|>r}
    |\hat f(w)| \cdot \|w\|
    \dif w.
  \]
\end{proof}

With more care (in particular, a crucial change of variable), a much better bound
is possible for convolutions with Gaussians.

\begin{proof}[Proof of \Cref{fact:transport:fourier:gaussian}]
  By \Cref{fact:fourier:helper}, setting $\phi := (2\pi \sigma)^{-1}$,
  \[
    |\hat f_\alpha(w) |
    = | \hat f(w) | \hat G_\alpha(w)
    = (2\pi)^{d/2} | \hat f(w) | G(w/\phi).
  \]
  Plugging this into \Cref{fact:fourier}
  and again defining $h(z) := \cos(2\pi z)$ for convenience,
  but unlike \Cref{fact:transport:fourier} performing a change of variable
  to directly introduce $G(\tw)$, and then manually introducing $G(b)$,
  \begin{align*}
    &\hspace{-2em}
    f_\alpha(x) - f_\alpha(0) - \int |\hat f_\alpha(w) | h(\theta_{f_\alpha}(w) - \|w\|)\dif w
    \\
    &=
    - \iint |\hat f_\alpha(w)|
    h'(\theta_{f_\alpha}(w) - b)\sigma'(\tw^\T\tx) \1[\|w\|\geq |b|]
    \dif b \dif w
    \\
    &=
    - (2\pi)^{d/2}\iint |\hat f(w)| G(w/\phi)
    h'(\theta_{f_\alpha}(w) - b)\sigma'(\tw^\T\tx) \1[\|w\|\geq |b|]
    \dif b \dif w
    \\
    &=
    - (2\pi \phi^2)^{d/2}\phi \int |\hat f(\phi w)|G(w)
    h'(\theta_{f_\alpha}(\phi w) - b) \sigma'(\phi \ip{\tw}{\tx}) \1[\phi\|w\|\geq \phi |b|]
    \dif b \dif w
    \\
    &=
    - (2\pi \phi^2)^{(d+1)/2} \int |\hat f(\phi w)| e^{b^2/2}
    h'(\theta_{f_\alpha}(\phi w) - b) \sigma'(\ip{\tw}{\tx}) \1[\|w\|\geq |b|]
    \dif G(\tw).
  \end{align*}
  As in \Cref{fact:transport:fourier}, the transport is constructed
  by using $\bbE \sigma'(\ip{\tw}{\tx}) = \sfrac 1 2$ to model constants:
  $\cT_r(w,b) = (0,\ldots,0,p_r(\tw))$, where
  \begin{align*}
    p_r(\tw) &:= 
    2 \sbr{ 
      f_\alpha(0) + \int |\hat f_\alpha(v) | h(\theta_{f_\alpha}(v) - \|v\|)\dif v
    }
    \\
    &\quad
    - (2\pi \phi^2)^{(d+1)/2} |\hat f(\phi w)| e^{b^2/2}
    h'(\theta_{f_\alpha}(\phi w) - b) \1\sbr{|b| \leq \|w\|\leq r },
  \end{align*}
  with $f(x) = \bbE_\tw \ip{\cT_\infty(\tw)}{\Phi(x;\tw)}$ for $\|x\|\leq 1$ by construction.

  When $r<\infty$,
  by construction
  \begin{align*}
    \sup_{\tw} \|\cT_r(\tw)\|
    &\leq
    2 \envert{ f_\alpha(0) }
    + 2 \int |\hat f_\alpha(v) | \dif v
    + 2\pi (2\pi\phi^2)^{(d+1)/2}
    \sup_{\substack{\|w\|\leq r\\|b| \leq \|w\|}}
    |\hat f(\phi w) | e^{b^2/2},
  \end{align*}
  where $| \hat f(\phi w)|= 1$ when $f_\alpha = G_\alpha$ (meaning $f$ itself
  is the Dirac at $0$), and more generally \Cref{fact:fourier:helper} grants
  $|\hat f(\phi w)| \leq \|f(\phi \cdot)\|_{L_1}$;
  as in the lemma statement, these cases are summarized with $|\hat f(\phi w)| \leq M_f$.
  Plugging this in and simplifying further via \Cref{fact:fourier:helper},
  \begin{align*}
    \sup_{\tw} \|\cT_r(\tw)\|
    &\leq
    2 \envert{ \int f(x) G_\alpha(-x) \dif x }
    + 2 (2\pi)^{d/2} \int |\hat f(v)| G(v/\phi)\dif v
    + 2\pi (2\pi\phi^2)^{(d+1)/2}
    M_f 
    \sup_{|b| \leq r}
    e^{b^2/2}
    \\
    &\leq
    2 \sbr{
      M
      + 2 (2\pi\phi^2)^{d/2} M_f
      + 2\pi (2\pi\phi^2)^{(d+1)/2}
      M_f 
      \sup_{|b| \leq r}
      e^{b^2/2}
    }.
  \end{align*}

  For the approximation estimate, for any $\|x\|\leq 1$,
  the preceding derivation and \Cref{fact:gauss} grant
  \begin{align*}
    &
    \hspace{-2em}
    \envert{ f(x) - \bbE \ip{\cT_r(\tw)}{\Phi(x;\tw)} }
    \\
    &=
    \envert{ \bbE\ip{\cT_\infty(\tw) - \cT_r(\tw)}{\Phi(x;\tw)} }
    \\
    &\leq
    2\pi (2\pi\phi^2)^{(d+1)/2}
    \int_{\|w\|>r}\int_{|b|\leq \|w\|}
    |\hat f(\phi w)| |\sin(2\pi(w^\T x + \theta_{f_\alpha}(w))) \sigma'(\tw^\T \tx) |
    \dif b \dif G(w)
    \\
    &\leq
    2\pi (2\pi\phi^2)^{(d+1)/2}
    M_f
    \int_{\|w\|>r}\int_{|b|\leq \|w\|}
    \dif b \dif G(w)
    \\
    &\leq
    4\pi (2\pi\phi^2)^{(d+1)/2}
    M_f
    \int_{\|w\|>r}
    \|w\| \dif G(w)
    \\
    &\leq
    4\pi (2\pi\phi^2)^{(d+1)/2}
    M_f
    (\sqrt d + 3)
    \exp\del{ -(r-\sqrt{d})^2/4 }.
  \end{align*}
\end{proof}

\section{Deferred proofs from \Cref{sec:cont}}
\label{app:cont}

The first proof is of the approximation properties of Gaussian convolution;
as stated in the body, the proof proceeds by splitting the error into two terms,
one for nearby points, the other for distant points.

\begin{proof}[Proof of \Cref{fact:cont}]
  Splitting the integral into two terms, for any $\|x\|\leq 1$,
  \begin{align*}
    \envert{ f(x) - (f_{|\delta}*G_\alpha)(x) }
& =
    \envert{ \int f_{|\delta}(x) G_\alpha(z) \dif z  - \int f_{|\delta}(z) G_\alpha(x-z) \dif z }
    \\
    & =
    \envert{ \int f_{|\delta}(x) G_\alpha(z) \dif z  - \int f_{|\delta}(x-z) G_\alpha(z) \dif z }
    \\
    &\leq
    \int \envert{ f_{|\delta}(x) - f_{|\delta}(x-z)  } G_\alpha(z) \dif z
    \\
    &=
    \int_{\|z\|\leq \delta} \envert{ f_{|\delta}(x) - f_{|\delta}(x-z)  } G_\alpha(z) \dif z
    \\
    &\quad
    +
    \int_{\|z\|>\delta} \envert{ f_{|\delta}(x) - f_{|\delta}(x-z)  } G_\alpha(z) \dif z.
  \end{align*}
  Analyzing these terms separately, the definition of $\omega_f(\delta)$ gives
  \[
    \int_{\|z\|\leq \delta} \envert{ f_{|\delta}(x) - f_{|\delta}(x-z)  } G_\alpha(z) \dif z
    \leq
    \int_{\|z\|\leq \delta} \omega_f(\delta) G_\alpha(z) \dif z
    \leq
    \omega_f(\delta),
  \]
  whereas Gaussian concentration (cf. \Cref{fact:gauss}) gives
  \[
    \int_{\|z\|>\delta} \envert{ f_{|\delta}(x) - f_{|\delta}(x-z)  } G_\alpha(z) \dif z
    \leq 2 M \Pr[ \|\alpha z\| > \delta ]
    \leq 2 M \exp( - (\delta/\alpha - \sqrt{d})^2 / 2)
    \leq \omega_f(\delta).
  \]
\end{proof}

This now combines with \Cref{fact:transport:fourier:gaussian} to prove \Cref{fact:cont:transport}.

\begin{proof}[Proof of \Cref{fact:cont:transport}]
  Plugging the choice of $r$ into \Cref{fact:transport:fourier:gaussian},
  for any $\|x\|\leq 1$,
  \begin{align*}
    \envert{ f(x) - \bbE \ip{\cT_r(\tw)}{\Phi(x;\tw)} }
    &
    \leq
    4\pi (2\pi\phi^2)^{(d+1)/2}
    M_f
    (\sqrt d + 3)
    \exp\del{ -(r-\sqrt{d})^2/4 }
\leq \omega_f(\delta).
  \end{align*}
  Moreover, plugging $r$ into the estimate on $\sup_{\tw} \|\cT_r(\tw)\|$
  provided by \Cref{fact:transport:fourier:gaussian} gives
  \[
    \sup_{\tw} \|\cT_r(\tw)\|
    \leq
2 \sbr{
      M
      + (2\pi\phi^2)^{d/2} M_f
      \del{ 1 + \sqrt{2\pi^3 \phi^2} e^{r^2/2}
      }
    }
    ,
  \]
  where \Cref{fact:fourier:helper} and the choice of $r$ give
  \[
    M_f \leq \|f_{|\delta}\|_{L_1},
    \qquad\textup{and}\qquad
    e^{r^2/2}
    \leq
    e^d e^{(r-\sqrt{d})^2},
  \]
  where
  \[
    e^{(r- \sqrt d)^2}
    =
    \del{\frac {4\pi(2\pi\phi^2)^{(d+1)/2}M_f (\sqrt d + 3)}{\omega_f(\delta)}}^4
    =
    \cO\del{
      \del{\frac {M_f \sqrt d}{\omega_f(\delta) \alpha^{d+1}}}^4
    },
  \]
  and noting moreover that $\alpha = \tcO(\delta/\sqrt{d})$.
\end{proof}

To close this section comes the full version of \Cref{fact:cont:direct:simple},
which gives explicit constructions for both threshold $\sigma'$ and ReLU $\sigma$.
Interestingly, in the case of $\sigma'$, it is not necessary to truncate the density,
as is the case everywhere else in this work.

\begin{theorem}
  \label{fact:cont:direct}
  As in \Cref{fact:cont},
  let $f:\R^d\to\R$ and $\delta>0$ be given,
  and define
  \[
    M:=\sup_{\|x\|\leq 1+\delta} |f(x)|,
    \qquad
    f_{|\delta}(x) := f(x) \1[\|x\| \leq 1+\delta],
    \qquad
    \alpha := \frac {\delta}{\sqrt{d} + \sqrt{2 \ln(2M/\omega_f(\delta))}}.
  \]
  Let $G_\alpha$ denote a Gaussian with the preceding variance $\alpha^2$,
  and define $h := f_{|\delta}*G_\alpha$ with Fourier transform $\hat h$
  satisfying radial decomposition
  $\hat h(w) = |\hat h(w)| \exp(2\pi i \theta_h(w)$.
  Lastly, let $P$ be a probability measure supported on $\|x\|\leq 1$.
  \begin{enumerate}
    \item
      Additionally define
      \[
        c_1 := h(0) + \int |\hat h(w)| \cos\del{ 2\pi(\theta_h(w) - \|w\|)}\dif w,
        \qquad
        p_1 := 2\pi |\hat h(w)| \sin(2\pi(\theta_h(w) - b)) \1\sbr{ |b| \leq \|w\| }.
      \]
      Then
      \[
        |c_1| \leq M + \|f_{|\delta}\|_{L_1} (2\pi\alpha^2)^{d/2},
        \qquad\textup{and}\qquad
        \|p_1\|_{L_1} \leq 2\|f_{|\delta}\|_{L_1} \sqrt{\frac{2\pi d}{(2\pi\alpha^2)^{d+1}}},
      \]
      and
      with probability at least $1-3\eta$ over a draw of $((s_j,\tw_j))_{j=1}^m$ from
      $p_1$ (cf. \Cref{defn:signed_sampling}),
      \begin{align*}
        \enVert{
         f - \sbr{ c_1 + \|p_1\|_{L_1} \sum_{j=1}^m s_j \sigma'(\ip{\tw_j}{\cdot}) }
        }_{L_2(P)}
        &\leq
        2\omega_f(\delta)
        + \|p_1\|_{L_1} \sbr{\frac {1 + \sqrt{2\ln(1/\eta)}}{\sqrt m}},
        \\
        \sup_{\|x\|\leq 1}
        \envert{
         f(x) - \sbr{ c_1 + \|p_1\|_{L_1} \sum_{j=1}^m s_j \sigma'(\ip{\tw_j}{x}) }
        }
        &\leq
        2\omega_f(\delta) + \frac{\|p\|_{L_1}}{\sqrt m}
        \sbr{
          \sqrt{8 (d+1) \ln(m+1)} + \sqrt{\ln(1/\eta)}
        }.
      \end{align*}

    \item
      Additionally define
      \begin{align*}
        c_2
        &:= f(0) 
        f(0)
        + \int |\hat h(w) | \sbr{ \cos(2\pi(\theta_h(w)-\|w\|)) -2\pi \|w\|\sin(2\pi(\theta_h(w)-\|w\|)) } \dif w,
        \\
        a_2
        &:= \int w | \hat h(w)| \dif w,
        \\
        r_2
        &:=
        \sqrt{d} + 2 \sqrt{\ln{\frac{24 \pi^2 (\sqrt{d}+7)^2 \|f_{|\delta}\|_{L_1} }{\omega_f(\delta)}}},
        \\
        p_2(\tw)
        &:=
        4\pi^2 |\hat h(w)| \cos(2\pi(\|w\| - b)) \1[|b| \leq \|w\|\leq r_2],
      \end{align*}
      and for convenience create fake (weight, bias, sign) triples
      \begin{align*}
        (w,b,s)_{m+1}
:= (0, |c_2|, m\cdot\sgn(c_2)),
        \quad
(w,b,s)_{m+2}
:= (a_2, 0, +m),
        \quad
(w,b,s)_{m+3}
:= (-a_2, 0, -m).
      \end{align*}
      Then
      \begin{align*}
        \|a_2\|_2
        &\leq \sqrt{d} \|f_{|\delta}\|_{L_1} \phi (2\pi\alpha^2)^{-d/2},
        \\
        \|p_2\|_{L_1}
        &\leq 2\|f_{|\delta}\|_{L_1} \sqrt{\frac{(2\pi)^3 d}{(2\pi\alpha^2)^{d+1}}},
        \\
        |c_2|
        &\leq
        M + 2\sqrt{d} \|f_{|\delta}\|_{L_1} (2\pi\alpha^2)^{-d/2},
      \end{align*}
      and with probability at least $1-3\eta$ over a draw of $((s_j,\tw_j))_{j=1}^m$ from
      $p_2$ (cf. \Cref{defn:signed_sampling}),
      \begin{align*}
        \enVert{ f - \frac 1 m \sum_{j=1}^{m+3} s_j \sigma(\ip{\tw_j}{\cdot})}_{L_2(P)}
        &\leq
        3\omega_f(\delta)
        + r_2 \|p\|_{L_1} \sbr{\frac {1 + \sqrt{2\ln(1/\eta)}}{\sqrt m}},
        \\
        \sup_{\|x\|\leq 1} \envert{ f(x) - \frac 1 m \sum_{j=1}^{m+3} s_j \sigma(\ip{\tw_j}{\cdot})}
        &\leq
        3\omega_f(\delta) + \frac {4r_2\|p\|_{L_1}}{\sqrt m} \sbr{1 + \sqrt{\ln(1/\eta)}}.
      \end{align*}
  \end{enumerate}
\end{theorem}

\begin{proof}\begin{enumerate}
    \item
      By \Cref{fact:fourier} and the choice of $b_1$, for any $\|x\|\leq 1$,
      \[
        h(x) = c_1 + \int p_1(\tw) \sigma'(\ip{\tw}{x}) \dif \tw,
      \]
      thus by \Cref{fact:cont} and \Cref{fact:signed_maurey},
      defining $h_j := \|p\|_{L_1} \sigma'(\ip{\tw_j}{\cdot})$ for convenience,
      with probability at least $1-\eta$,
      \begin{align*}
        \enVert{ f - (c_1 + \sum_j h_j/m)}_{L_2(P)}
        &\leq
        \enVert{ f - h }_{L_2(P)} + \enVert{ h -  (c_1 + \sum_j h_j/m)}_{L_2(P)}
        \\
        &\leq
        2\omega_f(\delta) + \|p_1\|_{L_1} \sup_{\|\tw\|\leq r_2} \|\sigma'(\ip{\tw}{\cdot}\|_{L_2(P)}
        \sbr{\frac {1 + \sqrt{2\ln(1/\eta)}}{\sqrt m}}
        \\
        &\leq
        2\omega_f(\delta)
        + \|p_1\|_{L_1} \sbr{\frac {1 + \sqrt{2\ln(1/\eta)}}{\sqrt m}}.
      \end{align*}
      Similarly, the uniform norm bound follows by \Cref{fact:cont} and \Cref{fact:sample:uniform}:
      with probability at least $1-2\eta$, for any $\|x\|\leq 1$,
      \begin{align*}
        \envert{ f(x) - (c_1 + \sum_j h_j(x)/m)}
        &\leq
        \envert{ f(x) - h(x) } + \envert{ h(x) -  (c_1 + \sum_j h_j(x)/m)}
        \\
        &\leq
        2\omega_f(\delta) + \frac{\|p_1\|_{L_1}}{\sqrt m}
        \sbr{
          \sqrt{8 (d+1) \ln(m+1)} + \sqrt{\ln(1/\eta)}
        }.
      \end{align*}
      For the estimates on $|c_1|$ and $\|p_1\|_{L_1}$, note
      setting $\phi := (2\pi\alpha)^{-1}$,
      note by \Cref{fact:fourier:helper}
      and a change of variable $w\mapsto \phi w$
      and \Cref{fact:gauss}
      that
      \begin{align*}
        \|p_1\|_{L_1}
        &\leq 2\pi \int | \widehat{ f_{|\delta} * G_\alpha }(w) | \int \1[|b|\leq \|w\|] \dif b \dif w
        \\
        &\leq 4\pi \|f_{|\delta}\|_{L_1} \int \|w\| (2\pi\phi^2)^{d/2} G_\phi(w) \dif w
        \\
        &= 4\pi (2\pi)^{d/2} \|f_{|\delta}\|_{L_1} \int \|\phi w\| \phi^d G(w) \dif w
        \\
        &\leq 4\pi (2\pi)^{d/2}\phi^{d+1} \|f_{|\delta}\|_{L_1}
        \int \|w\| G(w) \dif w
        \\
        &\leq 4\sqrt{d} \pi (2\pi)^{d/2}\phi^{d+1} \|f_{|\delta}\|_{L_1},
        \\
        &\leq 2\|f_{|\delta}\|_{L_1} \sqrt{\frac{2\pi d}{(2\pi\alpha^2)^{d+1}}}.
      \end{align*}
      Similarly,
      \begin{align*}
        |c_1|
        &\leq M + \|f_{|\delta}\|_{L_1} \int \hat G_\alpha(w)\dif w
        \leq
        M + \|f_{|\delta}\|_{L_1} (2\pi\phi^2)^{d/2}.
      \end{align*}

    \item
      By \Cref{fact:fourier} and \Cref{fact:gauss}
      and the various chosen parameters, for any $\|x\|\leq 1$,
      \begin{align*}
        \envert{ b_2 + \ip{x}{a_2} + \int p_2(\tw) \sigma(\ip{\tw}{x}) \dif \tw - h(x)}
        &\leq 12 \pi^2\int_{\|w\|>r_2} \|w\|^2 |\hat h(w)| \dif w
        \\
&\leq 24 \pi^2 (\sqrt{d}+7)^2\exp(-(r_2-\sqrt{d})^2/4)
        \\
        &\leq \omega_f(\delta).
      \end{align*}
      Thus by \Cref{fact:cont} and \Cref{fact:signed_maurey},
      defining $h_j := \|p\|_{L_1} s_j \sigma(\ip{\tw_j}{\cdot})$ for convenience,
      with probability at least $1-\eta$,
      \begin{align*}
        \enVert{ f - \sum_{j=1}^{m+3} h_j/m}_{L_2(P)}
        &\leq
        \enVert{ f - h }_{L_2(P)}
        + \enVert{ f - (b_2 + (\cdot)^\T c_2 + \bbE_{p_2} s_1 h_1}
        + \enVert{ \bbE_{p_2} s_1 h_1 -  \sum_{j=1}^m s_j h_j/m}_{L_2(P)}
        \\
        &\leq
        3\omega_f(\delta) + \|p_2\|_{L_1} \sup_{\|\tw\|\leq r_2} \|\sigma(\ip{\tw}{\cdot}\|_{L_2(P)}
        \sbr{\frac {1 + \sqrt{2\ln(1/\eta)}}{\sqrt m}}
        \\
        &\leq
        3\omega_f(\delta)
        + 2 r_2 \|p\|_{L_1} \sbr{\frac {1 + \sqrt{\ln(1/\eta)}}{\sqrt m}}.
      \end{align*}
      Similarly, the uniform norm bound follows by \Cref{fact:cont} and \Cref{fact:sample:uniform}:
      with probability at least $1-2\eta$, for any $\|x\|\leq 1$,
      \begin{align*}
        \envert{ f(x) - \sum_{j=1}^{m+3} h_j(x)/m}
        &\leq
        \envert{ f(x) - h(x) }
        + \envert{ f(x) - (b_2 + x^\T c_2 + \bbE_{p_2} s_1 h_1(x)}
        + \envert{ \E_{p_2} s_1 h_1(x) -  \sum_{j=1}^m s_j h_j(x)/m}
        \\
        &\leq
        3\omega_f(\delta) + \frac {4r_2\|p_2\|_{L_1}}{\sqrt m} \sbr{1 + \sqrt{\ln(1/\eta)}}.
      \end{align*}
      For the estimates on $|c_1|$ and $\|p_1\|_{L_1}$, note
      setting $\phi := (2\pi\alpha)^{-1}$,
      note by \Cref{fact:fourier:helper}
      and a change of variable $w\mapsto \phi w$
      and \Cref{fact:gauss}
      that
      \begin{align*}
        \|p_2\|_{L_1}
        &\leq 4\pi^2 \int | \widehat{ f_{|\delta} * G_\alpha }(w) | \int \1[|b|\leq \|w\| \leq r_2] \dif b \dif w
        \\
        &\leq 8\pi^2 \|f_{|\delta}\|_{L_1} \int_{\|w\|\leq r_2} \|w\| (2\pi\phi^2)^{d/2} G_\phi(w) \dif w
        \\
        &= 8\pi^2 (2\pi)^{d/2} \|f_{|\delta}\|_{L_1} \int_{\|\phi w\|\leq r_2} \|\phi w\| \phi^d G(w) \dif w
        \\
        &\leq 8\pi^2 (2\pi)^{d/2}\phi^{d+1} \|f_{|\delta}\|_{L_1}
        \int_{\|\phi w\| \leq r_2} \|w\| G(w) \dif w
        \\
        &\leq 8\sqrt{d}\pi^2 (2\pi)^{d/2}\phi^{d+1} \|f_{|\delta}\|_{L_1}
        \\
        &
        \leq 2\|f_{|\delta}\|_{L_1} \sqrt{\frac{(2\pi)^3 d}{(2\pi\alpha^2)^{d+1}}}.
      \end{align*}
      Similarly,
      \begin{align*}
        |c_2|
        &\leq M + \|f_{|\delta}\|_{L_1} \int (1 + \|w\|) \hat G_\alpha(w)\dif w
        \\
        &\leq
        M + \|f_{|\delta}\|_{L_1} (2\pi\phi^2)^{d/2} \int (1 + \|\phi w\|) \dif G(w)
        \\
        &\leq
        M + 2\sqrt{d} \|f_{|\delta}\|_{L_1} (2\pi\phi^2)^{d/2},
        \\
        \|a_2\|_2
        &= \enVert{ \int w |\hat h (w) | \dif w }
        \\
&\leq \int \|w\| |\hat h (w) | \dif w
        \\
        &\leq \|f_{|\delta}\|_{L_1} (2\pi\phi^2)^{d/2} \int \|\phi w\| |\hat h (w) | \dif w
        \\
        &\leq \sqrt{d} \|f_{|\delta}\|_{L_1} \phi (2\pi\phi^2)^{d/2}.
      \end{align*}

  \end{enumerate}
\end{proof}

\section{Deferred proofs from \Cref{sec:rkhs}}

Lastly, the two short proofs leading to the RKHS bounds.

\begin{proof}[Proof of \Cref{fact:rkhs}]
  \begin{enumerate}
    \item
      By Markov's inequality
      \[
        \int \1[ \|\cT(\tw)\|_2 > B ] \dif G(\tw)
        \leq \frac {\|\cT\|_\cH^2}{B^2},
      \]
      thus by Cauchy-Schwarz, for any $\|x\|\leq1$,
      \begin{align*}
        \envert{
          \ip{\cT_B}{\Phi(x;\cdot)}_\cH
          -
          \ip{\cT}{\Phi(x;\cdot)}_\cH
        }
        &=
        \envert{
          \int \del{\cT_B(\tw)-\cT(\tw)}^\T \tx \sigma'(\ip{\tw}{\tx})\dif G(\tw)
        }
        \\
        &\leq
        \int
        \envert{
          \1\sbr{\|\cT(\tw)\|_2>B}
        }
        \cdot
        \envert{
          \cT(\tw)^\T\tx
        }
        \cdot
        \envert{ \sigma'(\ip{\tw}{\tx}) }\dif G(\tw)
        \\
        &\leq
        \sqrt{
          \int
          \1\sbr{\|\cT(\tw)\|_2>B}^2
          \dif G(\tw)
          }\cdot\sqrt{
          \int
          \del{\cT(\tw)^\T\tx}^2
          \dif G(\tw)
        }
        \\
        &\leq
        \frac {\|\cT\|_\cH}{B}
        \cdot
        \sqrt{
          \int
          2 \|\cT(\tw)\|^2
          \dif G(\tw)
        }
        \\
        &=
        \frac {\|\cT\|_\cH^2\sqrt{2}}{B}.
      \end{align*}

    \item
      Proceeding similarly,
      but now using \Cref{fact:gauss} to control the indicator,
      \begin{align*}
        \envert{
          \ip{\cT_r}{\Phi(x;\cdot)}_\cH
          -
          \ip{\cT}{\Phi(x;\cdot)}_\cH
        }
        &=
        \envert{
          \int \del{\cT_r(\tw)-\cT(\tw)}^\T \tx \sigma'(\ip{\tw}{\tx})\dif G(\tw)
        }
        \\
        &\leq
        \int
        \envert{
          \1\sbr{\|\tw\|_2>r}
        }
        \cdot
        \envert{
          \cT(\tw)^\T\tx
        }
        \cdot
        \envert{ \sigma'(\ip{\tw}{\tx}) }\dif G(\tw)
        \\
        &\leq
        \sqrt{
          \int
          \1\sbr{\|\tw\|_2>r}^2
          \dif G(\tw)
          }\cdot\sqrt{
          \int
          \del{\cT(\tw)^\T\tx}^2
          \dif G(\tw)
        }
        \\
        &\leq
        \sqrt{ \exp\del{- (r-\sqrt{d})^2 / 2 } }
        \cdot
        \sqrt{
          \int
          2 \|\cT(\tw)\|^2
          \dif G(\tw)
        }
        \\
        &=
        \frac {\|\cT\|_\cH\sqrt{2}}{e^{(r-\sqrt{d})^2/4}}.
      \end{align*}
  \end{enumerate}
\end{proof}

\end{document}